%% file: dueling-reward-bandits.tex
\documentclass{article}

\usepackage{microtype}
\usepackage{graphicx}
\usepackage{subfigure}
\usepackage{booktabs} % for professional tables

\usepackage{hyperref}
\usepackage{algorithm}
\usepackage{algorithmic}

\PassOptionsToPackage{noend}{algorithmic}

\usepackage{amsmath}
\usepackage{amssymb}
\usepackage{mathtools}
\usepackage{amsthm}

\theoremstyle{plain}
\newtheorem{theorem}{Theorem}[section]

\newtheorem{lemma}[theorem]{Lemma}
\newtheorem{corollary}[theorem]{Corollary}
\theoremstyle{definition}
\newtheorem{definition}[theorem]{Definition}

\theoremstyle{remark}

\input{_default_packages.tex}

\newcommand{\rew}{\text{\normalfont (R)}}
\newcommand{\duel}{\text{\normalfont (D)}}
\newcommand{\suff}{\text{\normalfont (Suff)}}

\newcommand{\rregret}{R_T^{\rew}}
\newcommand{\dregret}{R_T^{\duel}}
\newcommand{\regret}{R_T}

\newcommand{\exparm}{k_t^{\text{exp}}}

\newcommand{\cridus}{\text{\normalfont CR}}

\newcommand{\mab}{\textsc{MAB}\xspace}
\newcommand{\drbandit}{\textsc{DR-MAB}\xspace}
\newcommand{\elimfusion}{\textsc{ElimFusion}\xspace}
\newcommand{\defusion}{\textsc{DecoFusion}\xspace}

\usepackage{caption}
\usepackage{array}
\usepackage{eqparbox}
\let\oldalgorithmiccomment=\algorithmiccomment
\renewcommand\algorithmiccomment[1]{
  \textcolor{gray}{\#\ \oldalgorithmiccomment{#1}}
}

\newcommand{\LINECOMMENT}[1]{\textcolor{gray}{\#\ \textbf{#1}}}

\usepackage{multicol}

\title{Fusing Reward and Dueling Feedback in Stochastic Bandits}
\author{
  Xuchuang Wang$^1$, Qirun Zeng$^2$, Jinhang Zuo$^2$, Xutong Liu$^3$,
  \and
   Mohammad Hajiesmaili$^1$, John C.S. Lui$^4$, Adam Wierman$^5$\\
  \\
  $^1$University of Massachusetts, Amherst, MA, USA\\
  $^2$City University of Hong Kong, Hong Kong\\
  $^3$Carnegie Mellon University, Pittsburgh, PA, USA\\
  $^4$The Chinese University of Hong Kong, Hong Kong\\
  $^5$California Institute of Technology, Pasadena, CA, USA\\
}
\date{}  % Leave date empty for arXiv

\begin{document}

\maketitle




\begin{abstract}
  This paper investigates the fusion of absolute (reward) and relative (dueling) feedback in stochastic bandits,
  where both feedback types are gathered in each decision round.
  We derive a regret lower bound, demonstrating that an efficient algorithm may incur only the smaller among the reward and dueling-based regret for each individual arm.
  We propose two fusion approaches:
  (1) a simple elimination fusion algorithm that leverages both feedback types to explore all arms and unifies collected information by sharing a common candidate arm set,
  and (2) a decomposition fusion algorithm that selects the more effective feedback to explore the corresponding arms
  and
  randomly assigns one feedback type for exploration and the other for exploitation in each round.
  The elimination fusion experiences a suboptimal multiplicative term of the number of arms in regret due to the intrinsic suboptimality of dueling elimination.
  In contrast, the decomposition fusion achieves regret matching the lower bound up to a constant under a common assumption.
  Extensive experiments confirm the efficacy of our algorithms and theoretical results.
\end{abstract}

\input{sections/1_introduction.tex}

\input{sections/3_model.tex}

\input{sections/4_warm_up_algorithm.tex}

\input{sections/5_key_algorithm.tex}

\input{sections/6_experiments.tex}
\input{sections/7_conclusion.tex}

\bibliography{bibliography}
\bibliographystyle{icml2025}

%%%%%%%%%%%%%%%%%%%%%%%%%%%%%%%%%%%%%%%%%%%%%%%%%%%%%%%%%%%%%%%%%%%%%%%%%%%%%%%
%%%%%%%%%%%%%%%%%%%%%%%%%%%%%%%%%%%%%%%%%%%%%%%%%%%%%%%%%%%%%%%%%%%%%%%%%%%%%%%
% APPENDIX
%%%%%%%%%%%%%%%%%%%%%%%%%%%%%%%%%%%%%%%%%%%%%%%%%%%%%%%%%%%%%%%%%%%%%%%%%%%%%%%
%%%%%%%%%%%%%%%%%%%%%%%%%%%%%%%%%%%%%%%%%%%%%%%%%%%%%%%%%%%%%%%%%%%%%%%%%%%%%%%
\newpage
\appendix
\onecolumn

\input{sections/2_related_works.tex}

\input{sections/10_deferred_pseudo_code.tex}

\input{sections/8_lower_bound_proof.tex}
\input{sections/9_upper_bound_proof.tex}

% \clearpage

% \input{sections/11_rebuttal_icml_25.tex}

%%%%%%%%%%%%%%%%%%%%%%%%%%%%%%%%%%%%%%%%%%%%%%%%%%%%%%%%%%%%%%%%%%%%%%%%%%%%%%%
%%%%%%%%%%%%%%%%%%%%%%%%%%%%%%%%%%%%%%%%%%%%%%%%%%%%%%%%%%%%%%%%%%%%%%%%%%%%%%%

\end{document}

%% file: _default_packages.tex
\usepackage{natbib}

\usepackage{graphicx}
\usepackage{multirow}
\usepackage{booktabs}
\usepackage{threeparttable, array, float}
\usepackage{colortbl}

\usepackage{textcomp}

\usepackage{xcolor}

\usepackage{hyperref}
\definecolor{mydarkblue}{rgb}{0,0.08,0.45}
\definecolor{blueberry}{RGB}{4,51,255}
\hypersetup{ %
    pdftitle={},
    pdfsubject={},
    pdfkeywords={},
    pdfborder=0 0 0,
    pdfpagemode=UseNone,
    colorlinks=true,
    linkcolor=mydarkblue,
    citecolor=mydarkblue,
    filecolor=mydarkblue,
    urlcolor=mydarkblue,
}

\usepackage{amsmath,amsfonts,amssymb}
\usepackage{bm, bbm}
\usepackage{mathtools}
\usepackage{nicefrac}

\mathtoolsset{showonlyrefs,showmanualtags}

\newcommand{\overlabel}[1]{\overset{(\text{#1})}}

\makeatletter
\newcommand\addstarred[1]{%
    \expandafter\let\csname\string#1@nostar\endcsname#1%
    \edef#1{\noexpand\@ifstar\expandafter\noexpand\csname\string#1@star\endcsname\expandafter\noexpand\csname\string#1@nostar\endcsname}%
    \expandafter\newcommand\csname\string#1@star\endcsname%
}
\makeatother
\newcommand{\1}[1]{\mathbbm{1}{\{#1\}}}
\addstarred\1[1]{\mathbbm{1}{\left\{#1\right\}}}

\DeclarePairedDelimiter\abs{\lvert}{\rvert}

\DeclareMathOperator*{\argmax}{arg\,max}
\DeclareMathOperator*{\argmin}{arg\,min}

\renewcommand{\ge}{\geqslant}
\renewcommand{\le}{\leqslant}

\renewcommand{\P}{\mathbb{P}}

\newcommand{\E}{\mathbb{E}}

\newcommand{\type}[1]{Type-\uppercase\expandafter{\romannumeral#1}}

\DeclareMathOperator*{\KL}{{\normalfont KL}}
\DeclareMathOperator*{\kl}{{\normalfont kl}}

\usepackage{xspace}
\usepackage{comment}

\usepackage{fontawesome5}

\usepackage[colorinlistoftodos,prependcaption,textsize=tiny,textwidth=2cm,color=green]{todonotes}

%% file: sections/1_introduction.tex
% !TeX root = ../dueling-reward-bandits.tex
\section{Introduction}

\begin{figure*}[tp]
    \centering
    \begin{minipage}{0.5\textwidth}
        \centering
        \captionof{table}{Regret bounds for \drbandit
        }
        \label{tab:regret-bounds}
        \resizebox{\textwidth}{!}{
            \begin{tabular}{ll}
                \toprule
                \textbf{Algorithm}
                 & \textbf{Regret Bound}
                \\
                \midrule
                No Fusion$^\dagger$
                 & \(O\left( \sum_{k\neq 1} \frac{(\alpha\Delta_k^{\rew} + (1-\alpha)\Delta^{\duel}_k)\log T}{\min\{(\Delta_k^{\rew})^2, (\Delta^{\duel}_k)^2\}} \right)\)
                \\
                \elimfusion (Alg.~\ref{alg:elimination-fusion})
                 & \(O\left( \sum_{k\neq 1} \frac{(\alpha\Delta_k^{\rew} + (1-\alpha)\Delta^{\duel}_k)\log T}{\max\{(\Delta_k^{\rew})^2, (\Delta^{\duel}_k)^2/K\}} \right)\)
                \\
                \defusion (Alg.~\ref{alg:defusion})
                 & \(O\left( \sum_{k\neq 1}\frac{\log T}{\max\{\Delta_k^{\rew}/\alpha, \Delta^{\duel}_k/(1-\alpha)\}} \right)\)
                \\
                Simplified LB (Cor.~\ref{cor:simplified-lower-bound})
                 & \(\Omega\left( \sum_{k\neq 1}\frac{\log T}{\max\{\Delta^{\rew}_k/\alpha, \Delta^{\duel}_k/(1-\alpha)\}} \right)\)
                \\
                \bottomrule
            \end{tabular}
        }
        \footnotesize{
            $^\dagger$~Two separate algorithms for reward and dueling feedback.}
    \end{minipage}
    \hfill
    \begin{minipage}{0.48\textwidth}
        \centering
        \resizebox{\textwidth}{!}{
            \begin{tikzpicture}
                \draw [<->, line width= 2, draw=blue] (0,0)--(8,0);

                \node at (0,-0.3) {\textcolor{blue}{\(\alpha=0\)}};
                \node at (8,-0.3) {\textcolor{blue}{\(\alpha=1\)}};

                \draw (0, 0.75) node[align=center]
                {\(O(\log T)\)};
                \draw (8, 0.75) node[align=center]
                {\(O(\log T)\)};

                \draw (4, 0.75) node[align=center]
                {\(O\left( \sum_{k\neq 1} \frac{(\alpha\Delta_k^{\rew} + (1-\alpha)\Delta^{\duel}_k)\log T}{\max\{(\Delta_k^{\rew})^2, (\Delta^{\duel}_k)^2/K\}} \right)\)};
                \draw (4, 1.4) node[align=center]
                {\textbf{\elimfusion} (Alg.~\ref{alg:elimination-fusion})};

                \draw (0, -1.25) node[align=center]
                {Constant};
                \draw (8, -1.25) node[align=center]
                {Constant};
                \draw (4, -1.25) node[align=center]
                {\(O\left( \sum_{k\neq 1}\frac{\log T}{\max\{\Delta_k^{\rew}/\alpha, \Delta^{\duel}_k/(1-\alpha)\}} \right)\)};
                \draw (4, -0.6) node[align=center]
                {\textbf{\defusion} (Alg.~\ref{alg:defusion})};

            \end{tikzpicture}}
        \captionof{figure}{Impact of parameter \(\alpha\) on regret \(\regret\)}
        \label{fig:alpha-impact-overview}
    \end{minipage}
\end{figure*}

\emph{Relative feedback} is a type of feedback that provides information about the relative quality of two or more items (i.e., which is better) rather than their \emph{absolute} quality.
The power of relative feedback has been widely recognized in various fields. For example, in the human alignment stage of the large language model (LLM) training~\citep{ouyang2022training,rafailov2024direct}, the human annotators are asked to compare the quality of two LLM-generated sentences, rather than scoring them using absolute values.
% \todo{and \tdmark~find some other applications, like human-in-the-loop etc.}
Given LLMs' popularity, designing algorithms that can effectively leverage relative feedback has become a critical research topic in the online learning community, e.g., reinforcement learning with human feedback (RLHF)~\citep{wang2023rlhf,xiong2024iterative},
bandits with human preferences~\citep{ji2023provable},
dueling bandits~\citep{yan2022human,saha2022versatile}, etc.
This line of work complements the traditional online learning algorithms that rely on absolute feedback, e.g., reward in bandits and reinforcement learning~\citep{lai1985asymptotically,szepesvari2022algorithms}.
We refer to Appendix~\ref{sec:related-works} for a detailed discussion on related works.

Most of the prior literature has mainly focused on algorithms that leverage either relative or absolute feedback alone. In many real-world applications, however, these two types of feedback coexist.  For example, in LLM training, the human annotators may provide both absolute scores and relative comparisons~\citep{ouyang2022training}.
Another example is in recommendation systems, where the user feedback can be both absolute ratings and pairwise comparisons (choose one among two recommendations)~\citep{zhang2020conversational}. Despite its practical relevance, there is no prior literature where the underlying algorithms simultaneously utilize (dubbed as \emph{fuse}) the absolute and relative feedback.

In this paper, we investigate how to fuse the absolute and relative feedback under the framework of
% \todo{or ``under the framework of'', and revise the paper title accordingly} 
the stochastic multi-armed bandit (\mab)~\citep{lattimore2020bandit}, which is a fundamental problem in online learning.
To align with the terminology in bandit literature, we refer to the absolute feedback as \emph{reward} and the relative feedback as \emph{dueling} feedback.
We consider a \mab with \(K \in \mathbb{N}^+\) arms, where each arm \(k\) is associated with a reward distribution with an unknown mean,
and each arm pair \((k_1,k_2)\) is associated with an unknown dueling probability, representing the probability that arm \(k_1\) is preferred over arm \(k_2\) in a comparison.
% Especially, if arm \(k_1\) is associated with a higher reward mean than arm \(k_2\), i.e., \(\mu_{k_1} > \mu_{k_2}\), then arm \(k_1\) also has a higher winner probability, i.e., \(v_{k_1, k_2} > 0.5\).
In each round \(t\), the learner selects a tuple of arms \(\{k_t, (k_{1, t}, k_{2, t})\}\), where the reward feedback is the observed reward drawn from the reward distribution of arm \(k_t\), and the dueling feedback is the observed winning arm from the comparison between arms \(k_{1, t}\) and \(k_{2, t}\).

The goal of the learner is to minimize the accumulative \emph{regret} \(\regret\) over \(T\in \mathbb{N}^+\) rounds.
The regret attributes to reward and dueling feedback, called reward-based regret \(\rregret\) and dueling-based regret \(\dregret\).
The \(\rregret\) (resp. \(\dregret\)) measures the difference between the reward (resp., dueling probability) of the selected arms (resp., arm pairs) and that of the optimal arm (resp., optimal arm pair) in hindsight.
One notable advantage of the dueling feedback is that querying relative feedback from suboptimal arm pairs is often more cost-efficient than that of the absolute feedback~\citep{ouyang2022training}.
Taking this cost difference into account, we introduce a weight parameter \(\alpha \in [0,1]\) and define the final regret as \(\regret \coloneqq \alpha \rregret + (1-\alpha)\dregret\).
We call this model \emph{dueling-reward multi-armed bandit} (\drbandit) and present its details in Section~\ref{sec:model}.
In this paper, we aim to answer the following central questions:

\emph{
    How do we fuse reward and dueling feedback in \drbandit so as to reduce the accumulative regret?
    How does the weight factor \(\alpha\) influence the fusion process and the regret?
}

\textbf{Technical challenges.} One key challenge comes from the orthogonal nature of reward and dueling feedback, each providing information on one aspect of the arms (i.e., reward means vs. dueling probabilities), both of which are incomparable.
That is, the heterogeneity of both feedback types makes it difficult to fuse them.
Furthermore, the relative costs of reward and dueling feedback (parameter \(\alpha\) in the regret) further complicate the fusion process.
For example,
while in general, the learner needs to balance the utilization of two feedback types to minimize regret,
there exist cases where one feedback is much cheaper than the other, making it more beneficial to rely solely on the cheaper one.

% A sophisticated learning algorithm should be able to not only balance the trade-off between these two feedback types when the relative costs are moderate but also fully take advantage of some extreme cases when one feedback type is much cheaper than the other, invalidating the trade-off.

\subsection{Contributions}
This paper investigates the fusion of reward and dueling feedback in stochastic \mab setting. We summarize all our technical contributions in Table~\ref{tab:regret-bounds}. Below are the details of our technical contributions. 

We formulate the dueling-reward \mab problem and study its regret lower bound in Section~\ref{sec:model}.
We first provide a general lower bound and then, with an additional assumption, derive a simplified version as \(\Omega( \sum_{k}
\min\{\alpha/ \Delta^{\rew}_k, (1-\alpha) / \Delta^{\duel}_k\}\log T )\)\footnote{Most big-O and big-\(\Omega\) formulas in the introduction are informally tailored for readability. We refer the readers to corresponding theorems for their formal expressions.}, where the \(\Delta^{\rew}_k\) and \(\Delta^{\duel}_k\) are the reward and dueling gaps of arm \(k\) (definition in Section~\ref{sec:model}).
This bound highlights the benefit of fusing reward and dueling feedback: for each suboptimal arm \(k\), among the two weighted regret costs induced from reward and dueling feedback, fusion makes it possible to only pay the smaller one.

We first propose a simple elimination fusion (\elimfusion) algorithm in Section~\ref{sec:elimination-fusion}.
\elimfusion applies the elimination algorithms for stochastic and dueling bandits separately to both types of feedback and then fuses their information by sharing two algorithms' candidate arm sets (arms not identified as suboptimal yet). \elimfusion enjoys the benefit of fusion suggested by the lower bound and achieves a regret bound of \(O( \sum_{k} (\alpha \Delta_k^{\rew} + (1-\alpha)\Delta_k^{\duel})
\min\{1 / (\Delta^{\rew}_k)^2, K / (\Delta^{\duel}_k)^2\}\log T  )\),
% \todo{expressions here and in the table and figure are different, need to make them consistent.} 
where the regret cost of dueling feedback is suboptimal in terms of the factor \(K\), which is inherited from the intrinsic suboptimality of the dueling elimination algorithm.

We then devise a decomposition fusion (\defusion) algorithm in Section~\ref{sec:defusion}.
Notice that the lower bound suggests a decomposition of the suboptimal arms into two subsets:
one with the smaller regret cost induced by reward feedback and the other smaller by dueling feedback.
\defusion leverages this insight by approximating this arm decomposition.
However, without the knowledge of the reward and dueling gaps, the approximated decomposition may deviate from the ground-truth one, which we further address by proposing a novel randomized decision-making strategy that explicitly separates the exploration and exploitation.
Putting all together, \defusion achieves a regret bound of \(O( \sum_{k}
\min\{\alpha/ \Delta^{\rew}_k, (1-\alpha) / \Delta^{\duel}_k\} \log T)\), which matches the simplified lower bound (instead of the general form) and outperforms \elimfusion.
% The regret of \defusion does not match the general lower bound in Theorem~\ref{thm:regret-lower-bound}.

Furthermore, as Figure~\ref{fig:alpha-impact-overview} shows, \defusion is able to fully utilize a free exploration property when either \(\alpha = 0\) (free reward) or \(\alpha = 1\) (free dueling), to achieve a constant (i.e., independent from \(T\)) regret, which is impossible for \elimfusion.
Lastly, we conduct simulations to evaluate the performance of the proposed algorithms in Section~\ref{sec:experiments}.
% \todo{add details after finishing experiments.}
% the extreme cases where one feedback type is much cheaper than the other, and the
% algorithm is robust to the parameter \(\alpha\), which is crucial for the fusion process in practice.

% relies on either reward or dueling feedback, which is suboptimal in general.
% \todo{\tdmark~some more words on Figure~\ref{fig:alpha-impact-overview}?}

%% file: sections/3_model.tex
% !TeX root = ../dueling-reward-bandits.tex
\section{Model Formulation}\label{sec:model}

Consider a dueling-reward multi-armed bandit (\drbandit) with \(K\in \mathbb{N}^+\) arms. Each arm \(k\in\mathcal{K}\coloneqq\{1,2,\dots,K\}\) is associated with a Bernoulli\footnote{The choice of Bernoulli distributions is mainly for simplicity. One can extend the assumption to distributions with bounded interval support, i.e., \([0,1]\)-bounded, via more sophisticated analysis.} reward distribution with unknown mean \(\mu_k \in (0,1)\).
Each arm pair \((k,\ell)\in \mathcal{K}^2\) is associated with a dueling probability \(\nu_{k,\ell}\in (0,1)\),
representing the probability that arm \(k\) wins arm \(\ell\) in a duel.
Especially, \(\nu_{k,\ell} > 0.5\) if \(\mu_k > \mu_{\ell}\) and \(\nu_{k,\ell} = 0.5\) if \(\mu_k = \mu_{\ell}\).
and arm pairs \((k,\ell)\) and \((\ell, k)\) are symmetric, i.e., \(\nu_{k,\ell} = 1 - \nu_{\ell,k}\). Besides this ordering relation between reward means \(\mu_k\) and dueling probabilities \(\nu_{k,\ell}\), we do not assume or utilize any other specific relation between \(\mu_k\) and \(\nu_{k,\ell}\).
For simplicity, we assume all reward means are distinct, and
these arms are labeled in descending order regarding their reward means, i.e.,
\(\mu_1 > \mu_2 > \dots > \mu_K\). Therefore, arm \(1\) is the unique \emph{optimal arm} and the \emph{Condorcet winner} (i.e., \(\nu_{1,k}> 0.5\) for any arm \(k>1\))~\citep{urvoy2013generic}.

Consider \(T\in\mathbb{N}^+\) rounds in \drbandit.
In each round \(t \in \{1,2,\dots, T\}\), the learner picks a tuple \(\{k_t, (k_{1,t}, k_{2,t})\}\) consisting of an arm \(k_t\) and a pair of arms \((k_{1,t}, k_{2,t})\).
Then, the learner observes the reward realization \(X_{k_t,t}\) of the chosen arm \(k_t\) and the winner of the pair duel \(Y_{k_{1,t}, k_{2,t}, t}\), where \(X_{k_t,t}\) is sampled from the Bernoulli distribution with mean \(\mu_{k_t}\),
and \(Y_{k_{1,t}, k_{2,t}, t}\) is determined by a sample from the Bernoulli distribution with mean \(\nu_{k_{1,t}, k_{2,t}}\):
\(Y_{k_{1,t}, k_{2,t}, t} = k_{1,t}\) if the sample is \(1\), and \(Y_{k_{1,t}, k_{2,t}, t} = k_{2,t}\) if it is \(0\).
% and \(Y_{k_{1,t}, k_{2,t}, t}\) are sampled from the Bernoulli distributions with means \(\mu_{k_t}\) and \(\nu_{k_{1,t}, k_{2,t}}\), respectively, and especially, 
Both \(X_{k_t,t}\) and \(Y_{k_{1,t}, k_{2,t}, t}\) are independent across rounds and arms (pairs).
We call the \(X_{k_t, t}\) as the \emph{reward feedback} (absolute) and the \(Y_{k_{1,t}, k_{2,t}, t}\) as the \emph{dueling feedback} (relative), following the convention of bandit literature.

\textbf{Regret objective.}
The learner aims to minimize the accumulative \emph{regret} \(\regret\) over \(T\) rounds,
composed by the reward-based regret \(\rregret\)
and the dueling-based regret \(\dregret\).
Denote \(\Delta_k^{\rew} \coloneqq \mu_1 - \mu_k\)
and \(\Delta_k^{\duel} \coloneqq \nu_{1,k} - 0.5\) as the reward and dueling gaps between the optimal arm \(1\) and the suboptimal arm \(k\), respectively.
% The optimal tuple is \(\{1, (1,1)\}\) with zero reward and dueling gaps.
Then, \(\rregret\) is defined as the accumulation of the reward gaps \(\Delta_{k_t}^{\rew}\) of all chosen arms \(k_t\) over the \(T\) rounds,
while \(\dregret\) is defined as the accumulation of the average of the dueling gaps \({(\Delta_{k_{1,t}}^{\duel} + \Delta_{k_{2,t}}^{\duel})}/{2}\) of the picked arm pairs \((k_{1,t}, k_{2,t})\), i.e., \begin{align}
    \rregret & \coloneqq \sum_{t=1}^T \Delta_{k_t}^{\rew} = T\cdot \mu_1  - \sum_{t=1}^T \mu_{k_t},
    \\
    \dregret & \coloneqq \sum_{t=1}^T \frac{\Delta_{k_{1,t}}^{\duel} + \Delta_{k_{2,t}}^{\duel}}{2} =  \sum_{t=1}^T \frac{\nu_{1, k_{1,t}} + \nu_{1, k_{2,t}} - 1}{2}.
\end{align}
Lastly, we introduce a parameter \(\alpha \in [0,1]\) to balance the impact of the reward-based and the dueling-based regrets on the aggregated regret \(\regret\), defined as follows,
\begin{align}\label{eq:regret}
    \regret & \coloneqq \alpha\rregret + (1-\alpha) \dregret.
\end{align}

% \todo{consider do we need to leave a clue here for the cases of \(\alpha = 0\) and \(\alpha = 1\) of free exploration. [did in the end of lower bound section]}
% \begin{remark}
%     \textit{Pre-assigned probabilities based dueling:} \(\nu_{k_1, k_2}\in [0,1]\) for any two arms \(k_1, k_2 \in\mathcal{K}\).
%     \begin{itemize}
%         \item Common choice in dueling bandits literature
%         \item For example, \textit{reward mean based dueling:} by BTL model, we have \[
%                   \nu_{k,\ell} = \frac{\exp(\mu_{k})}{\exp(\mu_{k})  +\exp(\mu_{\ell})}.
%               \]
%         \item Utility-based dueling bandits~\citep{ailon2014reducing}: for each arm with utility value \(u_k \in (0, \frac{1}{2})\), we define
%               \[
%                   \nu_{k,\ell} \coloneqq \frac{1 - u_k - u_{\ell}}{2}.
%               \]
%     \end{itemize}
% \end{remark}

\subsection{Lower Bound}\label{subsec:lower-bound}

This section provides the regret lower bound for any consistent algorithm (Definition~\ref{def:consistent-algorithm}) in \drbandit.
We denote \(N_{k,t}\) as the number of times arm \(k\) is picked in the first \(t\) rounds for reward feedback, and \(M_{k,\ell,t}\) as the number of times the pairs \((k,\ell)\) and \((\ell, k)\) (due to their symmetry) are picked in the first \(t\) rounds for dueling feedback.

\begin{definition}[Consistent algorithm]\label{def:consistent-algorithm}
    An algorithm is called \emph{consistent} for \drbandit if for any suboptimal arm \(k\neq 1\) and parameter \(\gamma > 0\), it fulfills
    \(
    \E[N_{k,T}] = o(T^\gamma) \text{ and }
    \E[M_{k,\ell, T}] = o(T^\gamma)
    \) for any arm \(\ell \neq k\).
\end{definition}

The consistent definition covers all algorithms that achieve logarithmic regrets in \drbandit, e.g., UCB and elimination~\citep{auer2002using,auer2010ucb}.
This definition is a \emph{``generalization''} of the consistent policy in stochastic bandits~\citep{lai1985asymptotically}.
We first provide a lemma to bound the number of arm pulling and pair dueling.

\begin{lemma}\label{lma:information-lower-bound}
    For any suboptimal arm \(k\neq 1\), under any consistent algorithm, we have
    \begin{equation}\label{eq:information-lower-bound}
        \begin{split}
             & N_{k,T}\kl(\mu_k, \mu_1) +
            \sum_{\ell < k} M_{k,\ell, T} \kl\left(\nu_{k,\ell}, 0.5\right)
            \\
             & \qquad\qquad\qquad\qquad\qquad\qquad \ge  (1-o(1)) \log T,
        \end{split}
    \end{equation}
    where \(\kl(p,q) \coloneqq p \log \frac{p}{q} + (1 - p) \log \frac{1-p}{1-q}\) is KL-divergence between two Bernoulli distributions with means \(p\) and \(q\).
\end{lemma}

The proof of Lemma~\ref{lma:information-lower-bound} (formally in Appendix~\ref{app:proof-lower-bound}) is based on the information-theoretic lower bounds for MAB~\citep{lai1985asymptotically} and the dueling bandits~\citep{komiyama2015regret}. This proof needs a careful selection of pairs of the \drbandit instances and constructions of the key information quantities, so that the attributions from reward and dueling feedback would both appear in the LHS of~\eqref{eq:information-lower-bound}.

The LHS of~\eqref{eq:information-lower-bound} can be interpreted as the amount of ``information'' collected by the learner for arm \(k\) in the \(T\) rounds, and its two terms are ``information'' from the reward and dueling feedback.
% \(N_{k,T}\kl(\mu_k, \mu_1)\) is the reward-based information and the second term \(\sum_{\ell < k} M_{k,\ell, T} \kl\left(\nu_{k,\ell}, 0.5\right)\) is the dueling-based information.
Their additive relation in the LHS implies that the ``information'' can be attributed to two parts: the reward-based (first term) and the dueling-based (second term).
Especially, the second term for dueling-based information of suboptimal arm \(k\) is the sum over all pairs between the arm \(k\) and other better arms \(\ell\, (< k)\), implying that the \emph{effective} dueling-based information is collected by dueling with better arms.

Although Lemma~\ref{lma:information-lower-bound} suggests that the necessary exploration on a suboptimal arm \(k\) may be fulfilled by aggregating information collected from the reward and dueling feedback, we prove the following theorem to show that the potentially optimal way to minimize regret due to each suboptimal arm \(k\) is to focus on either the reward-based or the dueling-based exploration, \emph{exclusively}.

\begin{theorem}\label{thm:regret-lower-bound}
    For any consistent algorithm and regret balanced by \(\alpha \in [0,1]\), the following regret lower bound holds,
    \begin{equation}\label{eq:regret-lower-bound}
        \begin{split}
             & \liminf_{T\to\infty} \frac{\E[R_T]}{\log T} \ge
            \sum_{k\neq 1} \min\left\{  \frac{\alpha\Delta_k^{\rew}}{\kl(\mu_k, \mu_1)}
            , \min_{\ell < k} \frac{(1{-}\alpha) (\Delta^{\duel}_k {+} \Delta^{\duel}_{\ell})}{\kl\left( \nu_{k,\ell}, \frac 1 2 \right)}\right\}.
        \end{split}
    \end{equation}
\end{theorem}

A full proof is given in Appendix~\ref{app:proof-lower-bound} and is based on minimizing the regret decomposed in terms of the sampling times \(N_{k,T}\) and \(M_{k, \ell, T}\), given the constraints in Lemma~\ref{lma:information-lower-bound}.
As both the regret minimization objective and the constraints are linear expressions, the minimization can be solved by linear programming, yielding the minimal attained in one of two end nodes of the constraint line segment in~\eqref{eq:information-lower-bound}.

Theorem~\ref{thm:regret-lower-bound} provides a regret lower bound for any consistent algorithm in \drbandit.
The sum in the RHS of~\eqref{eq:regret-lower-bound} is over all suboptimal arms \(k\neq 1\).
% since the optimal arm \(1\) does not incur regret.
The two terms inside the outer minimization in the RHS correspond to the reward-based and dueling-based regrets for arm \(k\).
% , and the coefficients \(\alpha\) and \(1-\alpha\) are inherited from the definition of regret \(\regret\) in~\eqref{eq:regret}.
The outer minimization indicates that for arm \(k\), an effective algorithm should explore it through either the reward or dueling feedback, depending on which yields a smaller regret.

% (1) summation (2) two minimization, (3) parameter \(\alpha\)

The second term (inner minimization) in the RHS of~\eqref{eq:regret-lower-bound} implies that for each suboptimal arm \(k\), there exists a most effective arm (competitor) to duel with, denoted as \(\ell^*_k \coloneqq \min_{\ell < k} {(\Delta^{\duel}_k + \Delta^{\duel}_{\ell})}/{\kl\left( \nu_{k,\ell}, \frac 1 2 \right)}\).
% Denote \(\ell^*_k \coloneqq \min_{\ell < k} \frac{(\Delta^{\duel}_k + \Delta^{\duel}_{\ell})}{\kl\left( \nu_{k,\ell}, \frac 1 2 \right)}\) as the most effective arm to duel with arm \(k\). 
However, as discussed in~\citet[\S3.2]{komiyama2015regret}, the most effective arm \(\ell^*_k\) is usually the optimal arm, i.e., \(\ell^*_k = 1\), in many real world applications.
Then, by assuming \(\ell^*_k = 1\), the key term in the lower bound for any suboptimal arm \(k\) becomes
\(\min\{  {\alpha\Delta_k^{\rew}}/{\kl(\mu_k, \mu_1)}
, {(1-\alpha) \Delta^{\duel}_k }/{\kl\left( \nu_{k,1}, \frac 1 2 \right)}\}.\)
Further noticing the \(\kl(a, b) = \Theta((a-b)^2)\) property for the KL-divergence between two Bernoulli distributions~\citep[\S16]{lattimore2020bandit}, this term can be simplified as \(\min\{  {\alpha}/{\Delta_k^{\rew}}
, {(1-\alpha)}/{\Delta^{\duel}_k}\} = {1}/{\max\{ \Delta_k^{\rew}/\alpha,\Delta^{\duel}_k / (1-\alpha)\}}.\)
This simplification is summarized in the following corollary.

\begin{corollary}\label{cor:simplified-lower-bound}
    If for each suboptimal arm \(k\neq 1\), the most effective dueling arm \(\ell^*_k\) is the optimal arm \(1\), the regret lower bound in Theorem~\ref{thm:regret-lower-bound} can be simplified as follows, for some universal positive constant \(C\),
    \begin{align}\label{eq:simplified-lower-bound}
        \liminf_{T\to\infty}\! \frac{\E[R_T]}{\log T}\! \ge C \sum_{k\neq 1}\! \frac{1}{\max\{ \Delta_k^{\rew}{/}\alpha,\Delta^{\duel}_k {/} (1-\alpha)  \}}.
    \end{align}
\end{corollary}

From the regret lower bounds in~\eqref{eq:regret-lower-bound} and~\eqref{eq:simplified-lower-bound}, one may get a counter-intuitive observation: when \(\alpha = 0\) or \(1\), the regret lower bound would become sub-logarithmic for any consistent algorithm, which is unusual in the bandit literature.
Later in this paper, we will show that a sub-logarithmic \(o(\log T)\) regret is actually achievable (in fact, \(T\)-independent constant regret) in these two scenarios, revealing a unique phenomenon of \drbandit.

%% file: sections/4_warm_up_algorithm.tex
% !TeX root = ../dueling-reward-bandits.tex
\section{\elimfusion: Fusing Reward and Dueling via Elimination}\label{sec:elimination-fusion}

The estimations of reward means \(\mu_k\)
and dueling probabilities \(\nu_{k,\ell}\) are orthogonal as their observations are independently sampled from distributions with different parameters.
This orthogonality makes it difficult to directly combine these two types of feedback in online learning.

To address this challenge, in this section, we introduce our first approach to fusing reward and dueling feedback in \drbandit, named {Elimination Fusion} (\elimfusion).
Although \elimfusion has a suboptimal regret in \drbandit, it suggests a simple and effective way to fuse absolute and relative feedback.
As the ultimate goal is to address the fusion in general online learning problems beyond bandits, the intuitive high-level idea of \elimfusion may have a broader application to other learning settings.
% \todo{\tdmark~may defer this discussion to the corresponding section.}

% While this algorithm is not optimal in terms of some coefficients' dependence in regret, it serves as a stepping stone to understand the information fusion in \drbandit and provides insights and intuitions for understanding the advanced algorithm presented in Section~\ref{sec:defusion}.

\subsection{Algorithm Design of \elimfusion}

Arm elimination is a common technique in multi-armed bandits~\citep{auer2010ucb} and dueling bandits~\citep{saha2022versatile}. The main idea is maintaining a candidate arm set \(\mathcal{C}\), initialized as the full arm set \(\mathcal{K}\), and as the learning proceeds, one gradually identifies and eliminates suboptimal arms from the set \(\mathcal{C}\) based on the observed feedback, until only the optimal arm remains.
\elimfusion leverages the maintained candidate arm set as a bridge to connect the two types of feedback.
Specifically, \elimfusion maintains a single candidate arm set \(\mathcal{C}\), and arms in this set can be eliminated according to \emph{either} reward \emph{or} dueling feedback, whichever is first triggered.

\elimfusion in Algorithm~\ref{alg:elimination-fusion} starts with a warm-up phase (detained in Algorithm~\ref{alg:initial-phase} in Appendix~\ref{app:deferred_pseudo_code}) to initialize the reward and dueling statistics, where each pair of arms is dueled once, and each arm is queried for rewards in a round-robin manner in the same time.
In each time slots, \elimfusion picks the arm and arm pair with smallest number of pulls and dueling times, respectively, which uniformly explores arms (Lines~\ref{line:uniform-explore-reward-arm-selection}--\ref{line:uniform-duel-observation}).
With the new reward and dueling observations, \elimfusion updates the reward and dueling statistics for each arm pair and arm, respectively (Line~\ref{line:elimfusion-statistics-update}, detailed in Algorithm~\ref{alg:statistics-update} in Appendix~\ref{app:deferred_pseudo_code}).
Then, \elimfusion eliminates arms based on the observed feedback (Line~\ref{line:dual-elimination}).
To illustrate the arm elimination, we define
the confidence radius
\(
\cridus_{k,t}^{\rew}
\coloneqq \sqrt{{2\log (Kt/\delta)}/{N_{k,t}}}
\) and \(
\cridus_{k,\ell,t}^{\duel}
\coloneqq \sqrt{{2\log (Kt/\delta)}/{M_{k,\ell,t}}}
\)
% \begin{align}
%     \cridus_{k,t}^{\rew}
%      & \coloneqq \sqrt{{2\log (Kt/\delta)}/{N_{k,t}}},
%     \\
%     \cridus_{k,\ell,t}^{\duel}
%      & \coloneqq \sqrt{{2\log (Kt/\delta)}/{M_{k,\ell,t}}},
% \end{align}
for the estimated reward mean and dueling probability, where \(\delta>0\) is the input confidence parameter of \elimfusion.
Then, \elimfusion eliminates an arm \(k\) if \emph{either} its upper confidence bound (UCB) of reward mean estimate \(\hat\mu_{k, t+1} + \cridus_{k,t+1}^{\rew}\) is less than the lower confidence bound (LCB) of highest mean estimate \(\hat\mu_{\hat k_{t + 1}^{\rew},t+1} - \cridus_{\hat k_{t + 1}^{\rew},t+1}^{\rew}\) (the \(\hat k_{t+1}^{\rew}\) is the estimated optimal arm in Line~\ref{line:elimfusion-empirical-optimal-arm}) \emph{or} there exists an arm \(\ell\) in \(\mathcal{C}\) such that the UCB of its dueling probability \(\hat{\nu}_{k,\ell,t+1} + \cridus_{k,\ell,t+1}^{\duel}\) is less than \({1}/{2}\) (i.e., arm \(\ell\) outperforms arm \(k\)).
\elimfusion keeps eliminating arms via the above loop, and when only the optimal arm remains in \(\mathcal{C}\), i.e., \(\abs{\mathcal{C}}=1\), it switches to exploitation  (Lines~\ref{line:elimfusion-exploit}--\ref{line:elimfusion-exploit-pull}).

\begin{algorithm}[H]
    \caption{Elimination Fusion (\elimfusion)}
    \label{alg:elimination-fusion}
    \begin{algorithmic}[1]
        \STATE \textbf{Input:} Arm set \(\mathcal K\), confidence parameter \(\delta\)
        \STATE \textbf{Initialize:} Candidate arm set \(\mathcal{C}\gets \mathcal{K}\), time slot \(t\gets 0\), reward pull counter \(N_{k,t}\gets 0\) and reward estimate \(\hat{\mu}_{k,t}\gets 0\) for all arms \(k\in\mathcal{K}\),
        dueling pull counter \(M_{k,\ell,t} \gets 0\) and dueling estimate \(\hat{\nu}_{k,\ell,t}\gets 0\) for all pairs \((k,\ell)\in\mathcal{K}^2\)
        \STATE Warm-up (Algorithm~\ref{alg:initial-phase})
        % \STATE \textcolor{gray}{\textbf{Warm-up (Algorithm~\ref{alg:initial-phase})}}
        \FOR{each time slot \(t\)}
        \STATE \LINECOMMENT{\texttt{decision making:}}
        % \textcolor{gray}{\#\ {Decision Making}}
        \IF[uniform explore]{\(\abs{\mathcal{C}} > 1\)} \label{line:uniform-explore-if-condition}
        \STATE \(k_t \gets \argmin_{k\in\mathcal{C}} N_{k,t}\)
        \label{line:uniform-explore-reward-arm-selection}
        \STATE \((k_{t,1}, k_{t,2}) \gets \argmin_{(k,\ell)\in\mathcal{C}^2, k\neq \ell} M_{k,\ell,t}\)

        \STATE Pull arm \(k_t\) and observe arm reward \(X_{k_t, t}\)
        \STATE Duel \((k_{t,1}, k_{t,2})\) and observe winner \(Y_{k_{t,1},k_{t,2},t}\) \label{line:uniform-duel-observation}

        \STATE Statistics update (Algorithm~\ref{alg:statistics-update}) \label{line:elimfusion-statistics-update}

        \STATE \LINECOMMENT{\texttt{arm elimination:}}
        % \textcolor{gray}{\#\ {Elimination as follows,}}
        % \STATE cc \COMMENT{Elimination}

        \STATE \(\hat k_{t + 1}^{\rew} \gets \argmax_{\ell\in\mathcal{C}} \hat\mu_{\ell,t+1}\) \label{line:elimfusion-empirical-optimal-arm}

        \STATE\label{line:dual-elimination}\(\mathcal{C}\gets \mathcal{C} \setminus \bigg\{k\in\mathcal{C}:\text{ either }\quad \exists \ell\in \mathcal{C}{\setminus} \{k\},  \hat{\nu}_{k,\ell,t+1} {+} \cridus_{k,\ell,t+1}^{\duel}  < \frac{1}{2}\,\,\,\,\,\)
        \\\hfill \(\text{or } \hat\mu_{k, t+1} + \cridus_{k,t+1}^{\rew}  \le   \hat\mu_{\hat k_{t + 1}^{\rew},t+1} - \cridus_{\hat k_{t + 1}^{\rew},t+1}^{\rew}\bigg\}\)

        \ELSE[exploit when only the optimal arm left] \label{line:elimfusion-exploit}

        % \STATE \COMMENT{Only the optimal arm left}
        \STATE \(k\gets \) the only remaining arm in set \(\mathcal{C}\)
        \STATE Pull arm \(k\) for reward and duel pair \((k,k)\)
        \label{line:elimfusion-exploit-pull}
        \ENDIF

        \ENDFOR
    \end{algorithmic}
\end{algorithm}

% \tdmark~give details here!

\subsection{Regret Analysis of \elimfusion}

Theorem~\ref{thm:elimination-fusion} provides the regret upper bound of \elimfusion. Its proof is deferred to Appendix~\ref{app:proof-elimination-fusion}.

\begin{theorem}\label{thm:elimination-fusion}
    % The reward-based and dueling-based regrets of Algorithm~\ref{alg:elimination-fusion}, with a probability of \(1-\delta\), are bounded as follows,
    % \begin{align}
    %     \rregret & \le \sum_{k\neq 1} \frac{2\Delta_k^{\rew} \log (KT/\delta)}{\max\{(\Delta_k^{\rew})^2/4, (\Delta^{\duel}_k)^2/(k-1)\}},
    %     \\
    %     \dregret & \le \sum_{k\neq 1} \frac{2\Delta^{\duel}_k\log (KT/\delta)}{\max\{(\Delta_k^{\rew})^2/8, (\Delta^{\duel}_k)^2/(k-1)\}}.
    % \end{align}
    % Combining both will lead to the upper bound for \(\regret\).
    Letting \(\delta \gets 1/K^2T^2\) and taking the expectation, the regret upper bound of \emph{\elimfusion} can be upper bounded as follows,
    % The final regret upper bound can be simplified as 
    \begin{align}\label{eq:elimfusion-regret}
        \E[\regret] \le O\left( \sum_{k\neq 1} \frac{(\alpha\Delta_k^{\rew} + (1-\alpha)\Delta^{\duel}_k)\log T}{\max\{(\Delta_k^{\rew})^2, (\Delta^{\duel}_k)^2/K\}} \right).
    \end{align}
\end{theorem}

Applying elimination algorithms to reward and dueling feedback individually without sharing the candidate arm set yields the regret upper bound as follows, \begin{align}
    \E[\regret] \le O\left( \sum_{k\neq 1} \frac{(\alpha\Delta_k^{\rew} + (1-\alpha)\Delta^{\duel}_k)\log T}{\min\{(\Delta_k^{\rew})^2, (\Delta^{\duel}_k)^2/K\}} \right),
\end{align}
% \todo{replace this by the No Fusion algorithm in Table~\ref{tab:regret-bounds}? as this large \(K\) may be removed via better algorithm.} 
where the denominator is the minimum of the two gap-dependent terms and can be much worse than the maximum in the denominator in~\eqref{eq:elimfusion-regret} of \elimfusion.
This indicates the improvement of fusing the reward and dueling feedback via sharing the same candidate arm set in \elimfusion.

While the fusion via the candidate arm set in the elimination algorithm is effective and easy to implement, the \(1 / K\) factor in
the \((\Delta_k^{\duel})^2/K\) term of~\eqref{eq:elimfusion-regret} is \emph{suboptimal}, which is inherited from the suboptimality of the elimination algorithm in the dueling bandits literature~\citep[Remark~1]{saha2022versatile}.
Furthermore, the dependence on the parameter \(\alpha\) in~\eqref{eq:elimfusion-regret}, compared to the lower bound in~\eqref{eq:simplified-lower-bound}, is also suboptimal, reflecting the limitation of the elimination algorithm design for more sophisticated optimization in \drbandit.
In the next section, we introduce an advanced algorithm to address these limitations and achieve the optimal regret balance and dependence on the problem parameters.

% \todo{try to write a more illustrative paragraph to connect to the advanced algorithm design in the next section}

%% file: sections/5_key_algorithm.tex
% !TeX root = ../dueling-reward-bandits.tex
\section{\defusion: Fusing Reward and Dueling via Suboptimal Arm Decomposition}\label{sec:defusion}

This section presents a novel approach to fusing the reward and dueling feedback, called decomposition fusion (\defusion).
We first present the high-level ideas and technical challenges of \defusion in Section~\ref{subsec:defusion-high-level-idea}, and then provide its detailed algorithm design in Section~\ref{subsec:defusion-algorithm-detail}.
Finally, we present the theoretical regret upper bound of \defusion in Section~\ref{subsec:defusion-theoretical-analysis}.

\subsection{High-Level Ideas and Technical Challenges}\label{subsec:defusion-high-level-idea}

\textbf{Decomposed exploration for two set of suboptimal arms.} We denote \(\mathcal{K}^{\rew} \coloneqq \{k\in \mathcal{K}:
{\alpha\Delta_k^{\rew}}/{\kl(\mu_k, \mu_1)} <  \min_{\ell < k} {(1-\alpha) (\Delta^{\duel}_k + \Delta^{\duel}_{\ell})}/{\kl\left( \nu_{k,\ell}, \frac 1 2 \right)}
\}\) as the subset of suboptimal arms incurring smaller regret from reward feedback, and \(\mathcal{K}^{\duel} = \mathcal{K}\setminus\{\mathcal{K}^{\rew} \cup \{1\}\}\) as that of dueling feedback. Then, the RHS of regret lower bound in~\eqref{eq:regret-lower-bound} can be decomposed as
% \begin{equation}\label{eq:regret-lower-bound-rewrite}
% \begin{split}
\(\sum_{k\in \mathcal{K}^{\rew}}  {\alpha\Delta_k^{\rew}}/{\kl(\mu_k, \mu_1)} + \sum_{k\in \mathcal{K}^{\duel}}
\min_{\ell < k} {(1-\alpha) (\Delta^{\duel}_k + \Delta^{\duel}_{\ell})}/{\kl\left( \nu_{k,\ell}, \frac 1 2 \right)}.\)
%     \end{split}
% \end{equation}
% where \(\mathcal{K}^{\rew}\) and \(\mathcal{K}^{\duel}\) are defined as follows,
% \begin{align}\label{eq:reward-arm-set}
%     \mathcal{K}^{\rew} {\coloneqq}\! \left\{\!\! k\in\mathcal{K}\!:\!
%     \frac{\alpha\Delta_k^{\rew}}{\kl(\mu_k, \mu_1)} {<}  \min_{\ell < k}\! \frac{(1{-}\alpha) (\Delta^{\duel}_k {+} \Delta^{\duel}_{\ell})}{\kl\left( \nu_{k,\ell}, \frac 1 2 \right)} \!\!\right\},
%     \\\label{eq:dueling-arm-set}
%     \mathcal{K}^{\duel} {\coloneqq}\! \left\{\!\! k\in\mathcal{K}\!:\!
%     \frac{\alpha\Delta_k^{\rew}}{\kl(\mu_k, \mu_1)} {\ge}  \min_{\ell < k}\! \frac{(1{-}\alpha) (\Delta^{\duel}_k {+} \Delta^{\duel}_{\ell})}{\kl\left( \nu_{k,\ell}, \frac 1 2 \right)} \!\!\right\}.
% \end{align}
This lower bound decomposition implies that to minimize regret, arms in \(\mathcal{K}^{\rew}\) and \(\mathcal{K}^{\duel}\) should be explored by reward and dueling feedback, respectively.
%  called the \emph{reward arm set} and \emph{dueling arm set}.\todo{do we need this terminology?}
% We call the \(\mathcal{K}^{\rew}\) and \(\mathcal{K}^{\duel}\) as the \emph{reward arm set} and \emph{dueling arm set}, respectively.
Inspired by this observation,
an algorithm with this minimal regret also needs to decompose arms as above and explore them accordingly, which we call that the algorithm has a \emph{decomposed exploration policy}.
However, the above decomposition relies on the reward and dueling gaps \(\Delta_k^{\rew}\) and \(\Delta_k^{\duel}\), which are unknown a priori. This poses a significant challenge for designing a near-optimal algorithm.

Before delving into the details of our algorithm design, we introduce \emph{``empirical log-likelihoods''} as the measures of the amount of information collected for distinguishing arm \(k\) up to time slot \(t\) from the reward and dueling feedback,
\(
I_{k,t}^{\rew}
\coloneqq N_{k,t} \kl\left(\hat{\mu}_{k,t}, \max_{\ell\in \mathcal{K}}\hat{\mu}_{\ell,t}\right),
I_{k,t}^{\duel}
\coloneqq \sum_{\ell \in \mathcal{K}: \hat\nu_{k, \ell,t} < \frac{1}{2}} M_{k, \ell, t} \kl\left(\hat{\nu}_{k, \ell, t}, \frac{1}{2}\right),
\)
introduced by~\citet{honda2010asymptotically} and~\citet{komiyama2015regret} for devising optimal algorithms in stochastic and dueling bandits, respectively.
Their algorithms use the conditions of \( I_{k,t}^{\rew} \le \log t + f(K)\) and \(  I_{k,t}^{\duel} - \min_{\ell\in\mathcal{K}} I_{\ell,t}^{\duel} \le \log t + f(K)\) for choosing arms to explore, and the function \(f(K)\) is independent of time horizon \(T\) and determined later.
% \todo{clarify \(f(K)\)}

\textbf{Decomposition deadlock.} The information measures \( I_{k,t}^{\rew}\) and \( I_{k,t}^{\duel}\) cannot work with the decomposed exploration policy.
Because both definitions involve the full arm set \(\mathcal{K}\) and thus need accuracy estimates of all arms.
But under the decomposed arm exploration policy,
arms in \({\mathcal{K}}^{\rew}\) only have good estimates of reward mean \(\hat \mu_{k,t}\), while arms in \(\mathcal{K}^{\duel}\) only have good estimates of dueling probability \(\hat \nu_{k,\ell,t}\), as exploring these arms by feedback other than the corresponding one would incur redundant regrets. This is a ``decomposition deadlock'': applying the decomposed exploration policy makes the \(I_{k,t}^{\rew}\) and \(I_{k,t}^{\duel}\)-based optimal bandit algorithms design invalid, and deploying the bandit algorithm design makes the decomposed exploration policy invalid.
Similar ``deadlock'' challenges also exist when considering other bandit algorithms, e.g., KL-UCB~\citep{cappe2013kullback}.

\subsection{Algorithm Design of \defusion} \label{subsec:defusion-algorithm-detail}

To address the above challenges, we propose the decomposition fusion (\defusion) algorithm.
\defusion, presented in Algorithm~\ref{alg:defusion}, consists of three main components:
(i) decomposition arm set construction (Lines~\ref{line:defusion-decomposition-arm-set-update}--\ref{line:defusion-decomposition-arm-set-update-end}),
(ii) randomized decision making (Lines~\ref{line:defusion-randomized-decision-making}--\ref{line:defusion-randomized-decision-making-end}), and
(iii) exploration arm set update (Lines~\ref{line:defusion-exploration-arm-set-update}--\ref{line:defusion-exploration-arm-set-update-end}).
Specifically, for the challenge of unknown decomposition sets,
we conservatively maintain two sets \(\hat{\mathcal{K}}_t^{\rew}\) and \(\hat{\mathcal{K}}_t^{\duel}\) in terms of the collected information instead of the incurred regret (details in Section~\ref{subsubsec:decomposition-arm-set-construction}).
However, this set construction leads to a mismatch between the constructed and ground-truth decompositions, which our proposed randomized decision-making strategy in Section~\ref{subsubsec:randomized-decision-making}  can address.
Finally, for the ``deadlock'' challenge of the conflict between the decomposed exploration policy and the calculation of information measures \(I_{k,t}^{\rew}\) and \(I_{k,\ell,t}^{\duel}\), in Section~\ref{subsubsec:exploration-arm-set-construction}, we devise a novel approach to update the exploration arm set \(\mathcal{E}\) based on the revised definitions of the information measures \(\hat I_{k,t}^{\rew}\) and \(\hat I_{k,t}^{\duel}\) and the approximated arm sets \(\hat{\mathcal{K}}_t^{\rew}\) and \(\hat{\mathcal{K}}_t^{\duel}\).

\begin{algorithm}[H]
    \caption{\defusion: Decomposition Fusion}
    \label{alg:defusion}
    % \begin{multicols}{2}
    \begin{algorithmic}[1]
        \STATE \textbf{Input:} Arm set \(\mathcal K\), parameter \(\alpha\),
        function \(f(K)\)
        \STATE \textbf{Initialize:} Exploration arm set \(\mathcal{E}\gets \mathcal{K}\), time slot \(t\gets 0\), decomposition arm set \(\hat{\mathcal{K}}_t^{\rew}, \hat{\mathcal{K}}_t^{\duel} \gets \mathcal{K}\), \(\hat I_{k,t}^{\rew}, \hat I_{k,t}^{\duel} \gets 0\) for all arms \(k\), and other initializations in Algorithm~\ref{alg:elimination-fusion}

        \STATE Warm-up (Algorithm~\ref{alg:initial-phase}) \label{line:defusion-warm-up}

        \FOR{each time slot \(t\)}

        \STATE \LINECOMMENT{\texttt{decomposition arm set update:}}\label{line:defusion-decomposition-arm-set-update}

        \STATE \(\hat{\mathcal{K}}_{t}^{\duel} \gets \{k\in \mathcal{K}: \hat I_{k,t}^{\rew} \le \log t + f(K)\}\)
        % \\ \COMMENT{arms \emph{indistinguishable} by reward feedback}
        \label{line:defusion-dueling-set-construction}

        \STATE \(\hat{k}_{t}^{\duel} \gets \argmin_{k\in \hat{\mathcal{K}}_{t}^{\duel}}\hat I_{k,t}^{\duel}\)
        \label{line:defusion-dueling-empirical-optimal-arm}

        \STATE \(\hat{\mathcal{K}}_{t}^{\rew} \gets \big\{k\in \mathcal{K}: \hat I_{k,t}^{\duel} - \hat I_{\hat{k}_{t}^{\duel}, t}^{\duel} \le \log t + f(K)\big\}\)
        % \\ \COMMENT{arms \emph{indistinguishable} by dueling feedback}
        \label{line:defusion-reward-set-construction}

        \STATE \(\hat{k}_{t}^{\rew} \gets \argmax_{k\in \hat{\mathcal{K}}_{t}^{\rew}} \hat\mu_{k,t}\)
        \label{line:defusion-reward-empirical-optimal-arm}
        \label{line:defusion-decomposition-arm-set-update-end}

        \STATE \LINECOMMENT{\texttt{randomized decision making:}} \label{line:defusion-randomized-decision-making}
        \STATE Pick an arm \(\exparm\) from \(\mathcal{E}\) according to a fixed order
        \label{line:defusion-pick-exp-arm}

        % \COMMENT{e.g., arm with smallest index}
        \IF{\(\text{Uniform}[0,1] > \frac{\alpha^2}{\alpha^2 + (1-\alpha)^2}\)} \label{line:defusion-random-decision-choice}
        \STATE \COMMENT{reward explore, dueling exploit}

        \STATE \(k_t \gets \exparm\) and \((k_{1,t}, k_{2,t}) \gets (\hat{k}_t^{\rew}, \hat{k}_t^{\rew})\)

        \ELSE[dueling explore, reward exploit]

        \STATE \(\hat{\mathcal{O}}_{\exparm, t} {\gets} \{k\in\hat{\mathcal{K}}_t^{\duel}{\setminus} \{\exparm\}: \hat\nu_{\exparm, k, t} {\le} \frac{1}{2}\}\) \label{line:defusion-dueling-explore-set contruction}

        \IF{\(\hat k^{\duel}_t \in \hat{\mathcal{O}}_{\exparm, t}\) or \(\hat{\mathcal{O}}_{\exparm, t} = \emptyset\)} \label{line:defusion-dueling-comparison-arm-choose}
        \STATE \(k_t^{\text{duel}} \gets \hat k^{\duel}_t\)

        \ENDIF
        \STATE \textbf{else} \(\quad k_t^{\text{duel}} \gets \argmin_{k\in\hat{\mathcal{K}}_t^{\duel}\setminus \{\exparm\}} \hat\nu_{\exparm, k, t}\)
        \label{line:defusion-dueling-comparison-arm-choose-end}

        \STATE \(k_t \gets \hat{k}_t^{\duel}\) and \((k_{1,t}, k_{2,t}) \gets (\exparm, k_t^{\text{duel}})\)

        \ENDIF
        \STATE Pull arm \(k_t\) and observe arm reward \(X_{k_t, t}\)
        \STATE Duel \((k_{t,1}, k_{t,2})\) and observe winner \(Y_{k_{t,1},k_{t,2},t}\)
        \label{line:defusion-randomized-decision-making-end}

        \STATE \LINECOMMENT{\texttt{exploration arm set update:}}\label{line:defusion-exploration-arm-set-update}

        \STATE \(\mathcal{E}\gets \mathcal{E}\setminus \{\exparm\}\) \COMMENT{remove the explored arm}\label{line:defusion-remove-exp-arm}

        \STATE \(\mathcal{B} \gets \mathcal{B}
        \cup \{k\in \hat{\mathcal{K}}_{t}^{\rew}\setminus \mathcal{E}:\hat I_{k,t}^{\rew} \le \log t + f(K)\}
        \cup \{k\in \hat{\mathcal{K}}_{t}^{\duel}\setminus \mathcal{E}: \hat I_{k,t}^{\duel} - \hat I_{\hat{k}_{t}^{\duel},t}^{\duel} \le \log t + f(K)\}\)
        % \COMMENT{add new arms for future exploration}
        \label{line:defusion-add-new-arm}

        \IF[last \(\mathcal{E}\) was traversed]{\(\mathcal{E}\) is empty}

        \STATE \(\mathcal{E}\gets \mathcal{B}\) and \(\mathcal{B}\gets \emptyset\) \COMMENT{renew \(\mathcal{E}\)}
        \label{line:defusion-renew-exploration-arm-set}
        % \COMMENT{update exploration arm set}
        \ENDIF\label{line:defusion-exploration-arm-set-update-end}

        \STATE Statistics update (Algorithm~\ref{alg:statistics-update}) \label{line:defusion-statistics-update}

        \STATE \(\hat I_{k,t+1}^{\rew}
        \gets N_{k,t} \kl(\hat{\mu}_{k,t}, \max_{\ell\in \hat{\mathcal{K}}_{t}^{\rew}}\hat{\mu}_{\ell,t})\) \label{line:defusion-reward-information-update}

        \STATE \(\hat I_{k,t+1}^{\duel}
        \gets \sum_{\ell \in \hat{\mathcal{K}}_{t}^{\duel}: \hat\nu_{k, \ell,t} < \frac{1}{2}} M_{k, \ell, t} \kl(\hat{\nu}_{k, \ell, t}, \frac{1}{2})\)
        \label{line:defusion-dueling-information-update}
        \ENDFOR
    \end{algorithmic}
    % \end{multicols}
\end{algorithm}

\subsubsection{Decomposition Arm Set Construction}\label{subsubsec:decomposition-arm-set-construction}

After the warm-up (Line~\ref{line:defusion-warm-up}), in each time slot, \defusion constructs two arm sets, \(\hat{\mathcal{K}}_t^{\rew}\) and \(\hat{\mathcal{K}}_t^{\duel}\), as analogs (\emph{not} estimates) of the sets \(\mathcal{K}^{\rew}\) and \(\mathcal{K}^{\duel}\).
The reason that \(\hat{\mathcal{K}}_t^{\rew}\) and \(\hat{\mathcal{K}}_t^{\duel}\) are only analogs, not estimates, is because that
the lack of knowledge of reward and dueling gaps \(\Delta_{k}^{\rew}\) and \(\Delta_k^{\duel}\) makes it hard to decompose arms according to the weighted reward and dueling-based regrets.
Instead, we utilize the revised information measures \(\hat I_{k,t}^{\rew}\) and \(\hat I_{k,t}^{\duel}\) (defined in Section~\ref{subsubsec:exploration-arm-set-construction}) to construct the arm sets.
Because this information-based set construction does not consider the weight \(\alpha\), the constructed sets \(\hat{\mathcal{K}}_t^{\rew}\) and \(\hat{\mathcal{K}}_t^{\duel}\) mismatches the optimal decomposition sets \(\mathcal{K}^{\rew}\) and \(\mathcal{K}^{\duel}\), even when the information measures are accurate at the end of the time horizon; thus only analogs.
We also note that while the optimal decomposition \(\mathcal{K}^{\rew}\) and \(\mathcal{K}^{\duel}\),
together with the singleton of optimal arm \(\{1\}\), forms an exclusive partition of the arm set \(\mathcal{K}\),
the conservatively constructed sets \(\hat{\mathcal{K}}_t^{\rew}\) and \(\hat{\mathcal{K}}_t^{\duel}\)
often overlap, and their intersection \(\hat{\mathcal{K}}_t^{\rew}\cap \hat{\mathcal{K}}_t^{\duel}\) contains the optimal arm \(1\) in almost all time slots.

Specifically, \(
\hat{\mathcal{K}}_{t}^{\duel} \coloneqq \{k\in \mathcal{K}: \hat I_{k,t}^{\rew} \le \log t + f(K)\}
\) is defined in Line~\ref{line:defusion-dueling-set-construction},
where the condition \(\hat I_{k,t}^{\rew} \le \log t + f(K)\) implies the insufficient information on arm \(k\) from reward feedback.
This is a conservative approximation of the dueling arm set \(\mathcal{K}^{\duel}\):
as long as there is insufficient information from reward feedback to exclude an arm from exploration, that arm is included in dueling feedback's corresponding set.
The other set \(\hat{\mathcal{K}}_t^{\rew}\) is constructed similarly, as detailed in Line~\ref{line:defusion-reward-set-construction}.
Especially, we calculate the estimated optimal arms \(\hat k_t^{\duel}\) and \(\hat k_t^{\rew}\) among the approximated dueling and reward arm sets, respectively, in Lines~\ref{line:defusion-dueling-empirical-optimal-arm} and~\ref{line:defusion-reward-empirical-optimal-arm}.
Because arms outside the corresponding sets may not have good estimate accuracy and thus are not considered.

\subsubsection{Randomized Decision Making}\label{subsubsec:randomized-decision-making}

To tackle the mismatch between the constructed \((\hat{\mathcal{K}}_t^{\rew}, \hat{\mathcal{K}}_t^{\duel})\) and the ground truth \((\mathcal{K}^{\rew}, \mathcal{K}^{\duel})\)
and realize the decomposed exploration policy,
we devise a randomized decision-making strategy.
In each time slot, \defusion picks one arm \(\exparm\) from the exploration arm set \(\mathcal{E}\) according to some fixed order in Line~\ref{line:defusion-pick-exp-arm}, e.g., the arm with the smallest index in the set.
Then, \defusion randomly chooses either the reward or dueling feedback to explore the arm \(\exparm\) in Line~\ref{line:defusion-random-decision-choice}.
% decides to explore the arm \(\exparm\) by either the reward or dueling feedback, depending on the value of a random variable sampled from the uniform distribution \(\text{Uniform}[0,1]\) (Line~\ref{line:defusion-random-decision-choice}).
Specifically,
if the round's realization of the uniform distribution \(\text{Uniform}[0,1]\) is greater than a threshold \(\frac{\alpha^2}{\alpha^2 + (1-\alpha)^2}\),
then the algorithm explores the arm \(\exparm\) by the reward feedback and exploits the pair of the reward estimated optimal arm \((\hat k_t^{\rew}, \hat k_t^{\rew})\) by the dueling feedback;
otherwise, the algorithm explores the arm \(\exparm\) by the dueling feedback and exploits the dueling estimated optimal arm \(\hat k_t^{\duel}\) by the reward feedback.
When exploring \(\exparm\) via dueling feedback  ,
we follow the RMED1 algorithm of \citet{komiyama2015regret} for dueling bandits to pick the comparison arm \(k_t^{\text{duel}}\) (Lines~\ref{line:defusion-dueling-comparison-arm-choose}--\ref{line:defusion-dueling-comparison-arm-choose-end}).
One key difference is that our comparison arm \(k_t^{\text{duel}}\) is selected from the set \(\hat{\mathcal{K}}_t^{\duel}\) instead of the full arm set \(\mathcal{K}\).

% we construct a set \(\hat{\mathcal{O}}_{\exparm, t}\) of arms that are empirically better than \(\exparm\) by the dueling feedback in the dueling decomposition set \(\hat{\mathcal{K}}_t^{\duel}\) (Line~\ref{line:defusion-dueling-explore-set contruction}) and then set the comparison arm \(k_t^{\text{duel}}\) as the dueling estimated optimal arm \(\hat k_t^{\duel}\) when either \(\hat k^{\duel}_t \in \hat{\mathcal{O}}_{\exparm, t}\) or \(\hat{\mathcal{O}}_{\exparm, t} = \emptyset\); otherwise, set \(k_t^{\text{duel}}\) as the arm in \(\hat{\mathcal{K}}_t^{\duel}\) that has the smallest empirical probability of \emph{losing} against \(\exparm\).

Finally, the threshold \(\frac{\alpha^2}{\alpha^2 + (1-\alpha)^2}\) is chosen to compensate for the mismatch in the decomposition arm set construction.
Although chosen by careful derivations,
the threshold has an intuitive explanation: take individual squares on all terms in \(\frac{\alpha}{\alpha + (1-\alpha)} = \alpha\).
The square exponent can be regarded as the compensation for the mismatch between the optimal set \(\mathcal{K}^{\rew}\) (omit \(\mathcal{K}^{\duel}\)) according to the regret with a linear dependence on \(1 / \Delta_k^{\rew}\) and the constructed set \(\hat{\mathcal{K}}_t^{\rew}\) via the collected information \(\hat I_{k,t}^{\rew}\) with quadratic \((1 / \Delta_k^{\rew})^2\) dependence.

\subsubsection{Exploration Arm Set Construction}\label{subsubsec:exploration-arm-set-construction}

Besides sets \(\hat{\mathcal{K}}_t^{\rew}\) and \(\hat{\mathcal{K}}_t^{\duel}\), \defusion also maintains an exploration arm set \(\mathcal{E}\) and an auxiliary arm set \(\mathcal{B}\) for updating \(\mathcal{E}\).
In each time slot, the exploration arm set \(\mathcal{E}\) outputs an arm to explore and removes it after exploration (Line~\ref{line:defusion-remove-exp-arm}),
and the auxiliary arm set \(\mathcal{B}\) adds arms that need further exploration regarding either the reward or dueling feedback, as detailed in Line~\ref{line:defusion-add-new-arm}, and each arm in the set \(\mathcal{B}\) is unique and not duplicated.
When all arms in \(\mathcal{E}\) are explored, i.e., \(\mathcal{E} = \emptyset\), the algorithm renews the exploration arm set \(\mathcal{E}\) by the auxiliary arm set \(\mathcal{B}\) (Line~\ref{line:defusion-renew-exploration-arm-set}).

At the end of a time slot, \defusion updates the statistics via Algorithm~\ref{alg:statistics-update} (Line~\ref{line:defusion-statistics-update}) as well as the revised empirical log-likelihoods \(\hat I_{k,t}^{\rew}\) and \(\hat I_{k,t}^{\duel}\) in Lines~\ref{line:defusion-reward-information-update} and~\ref{line:defusion-dueling-information-update}.
The key difference between the revised empirical log-likelihoods \(\hat I_{k,t}^{\rew}\) and \(\hat I_{k,t}^{\duel}\) and the original ones \(I_{k,t}^{\rew}\) and \(I_{k,t}^{\duel}\) is that the revised ones only involve the arms in the approximated sets \(\hat{\mathcal{K}}_t^{\rew}\) and \(\hat{\mathcal{K}}_t^{\duel}\) instead of the full arm set \(\mathcal{K}\).
% These empirical log-likelihoods are updated based on the collected information up to time slot \(t\) and the estimated optimal arms \(\hat k_t^{\rew}\) and \(\hat k_t^{\duel}\) in the constructed arm sets \(\hat{\mathcal{K}}_t^{\rew}\) and \(\hat{\mathcal{K}}_t^{\duel}\).

\subsection{Regret Analysis of \defusion}\label{subsec:defusion-theoretical-analysis}

% We present the regret of \defusion in below.
\begin{theorem}[Regret upper bound of Algorithm~\ref{alg:defusion}]\label{thm:defusion-regret}
    For sufficiently small \(\xi > 0\), and a constant \(c>0\) depending on the bandit instance,
    \defusion has the regret bound,
    \begin{align}
         & \E[R_T] \le O(K^2) + O(\epsilon^{-2}) + O(e^{cK - f(K)}) +
        \\\label{eq:defusion-regret}
         & \sum_{k\neq 1}  \frac{(\Delta_k^{\rew}/\alpha + \Delta_k^{\duel}/(1{-}\alpha))((1{+}\epsilon)\log T + f(K))}{\max\left\{\kl(\mu_k, \mu_1) / \alpha^2, \kl(\nu_{k,1}, \frac{1}{2}) / (1 {-} \alpha)^2\right\}}.
    \end{align}
    Simplifying the above bound by letting \(T\to \infty\), \(f(K) = cK^{1+\xi}\), and noticing \(\kl(p,q)=\Theta((p-q)^2)\), we have \begin{equation}\label{eq:defusion-simplified-regret}
        \begin{split}
            \E[R_T] \le O\bigg(
            \sum_{k\neq 1}\frac{\log T}{\max\{\Delta_k^{\rew} / \alpha, \Delta_k^{\duel} / (1 - \alpha)\}}
            \bigg).
        \end{split}
    \end{equation}
\end{theorem}

% \todo{\tdmark~finish the detail proof of this bound.}
\begin{figure*}[tb]
    \centering
    \subfigure[Aggregated regret \(\regret\)]{
        \includegraphics[width=0.465\columnwidth]{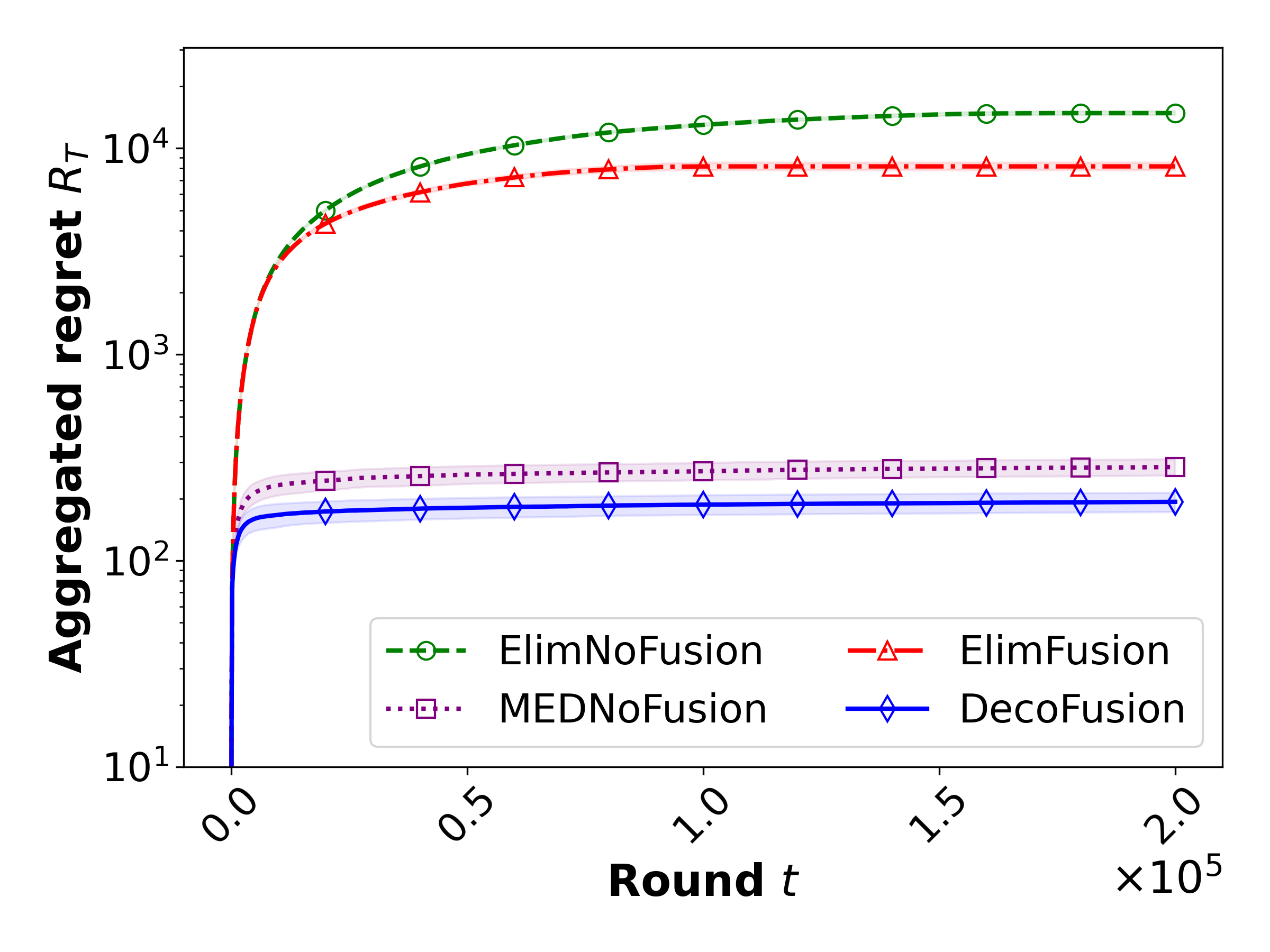}
        \label{fig:final-regret}
    }
    \hfill
    \subfigure[Vary reward gap \(\Delta_2^{\rew}\)]
    {
        \includegraphics[width=0.465\columnwidth]{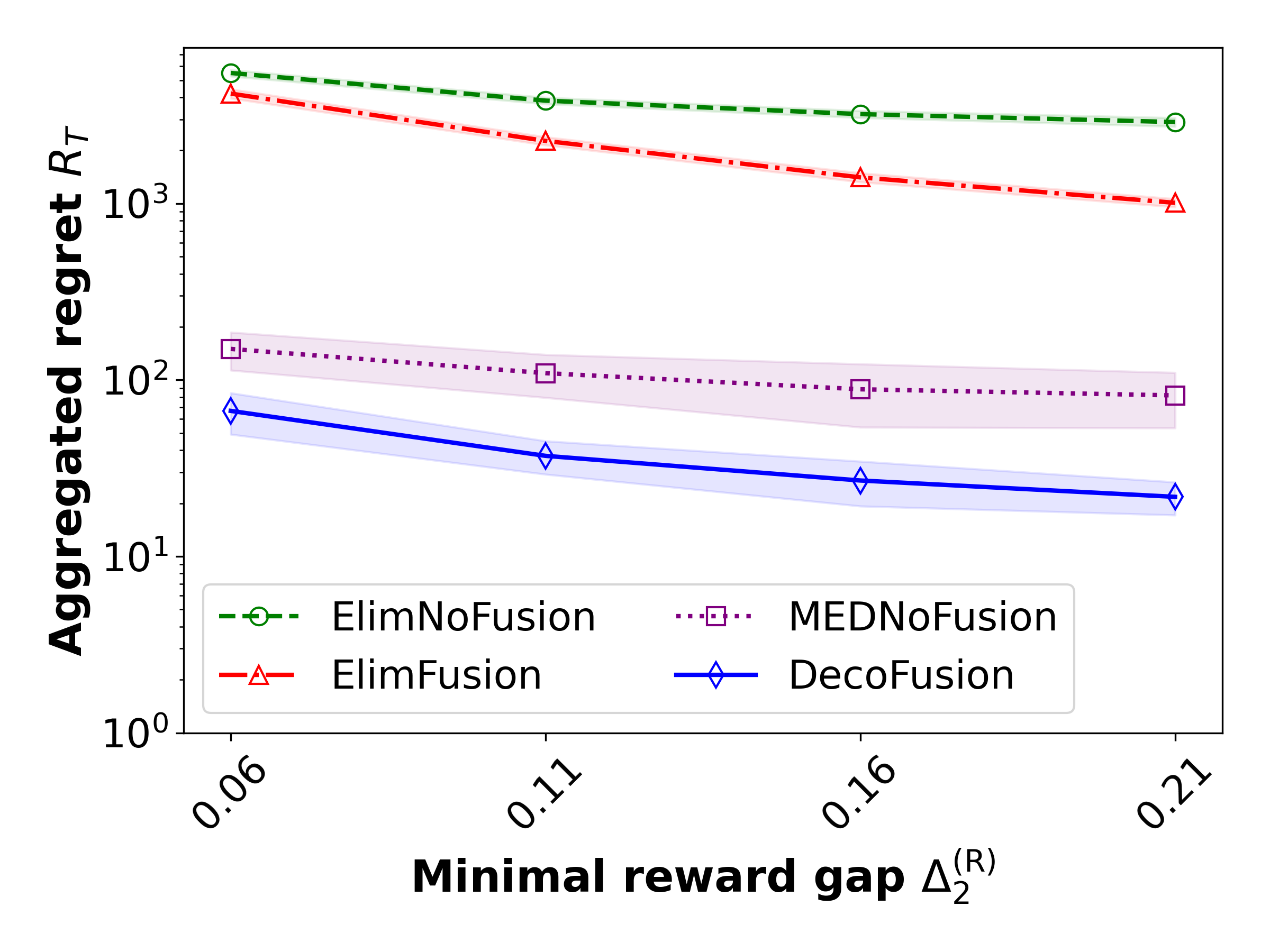}
        \label{fig:reward-gap}
    }
    % \tdmark~replace (b) and (c) as a table; (b) \& (c) for alpha = 0.9, 0.1; rounds to round; \(\Delta\) and ; change the ratio of a,b,c
    \hfill
    \subfigure[Vary dueling gap \(\Delta_2^{\duel}\)]{
    \includegraphics[width=0.465\columnwidth]{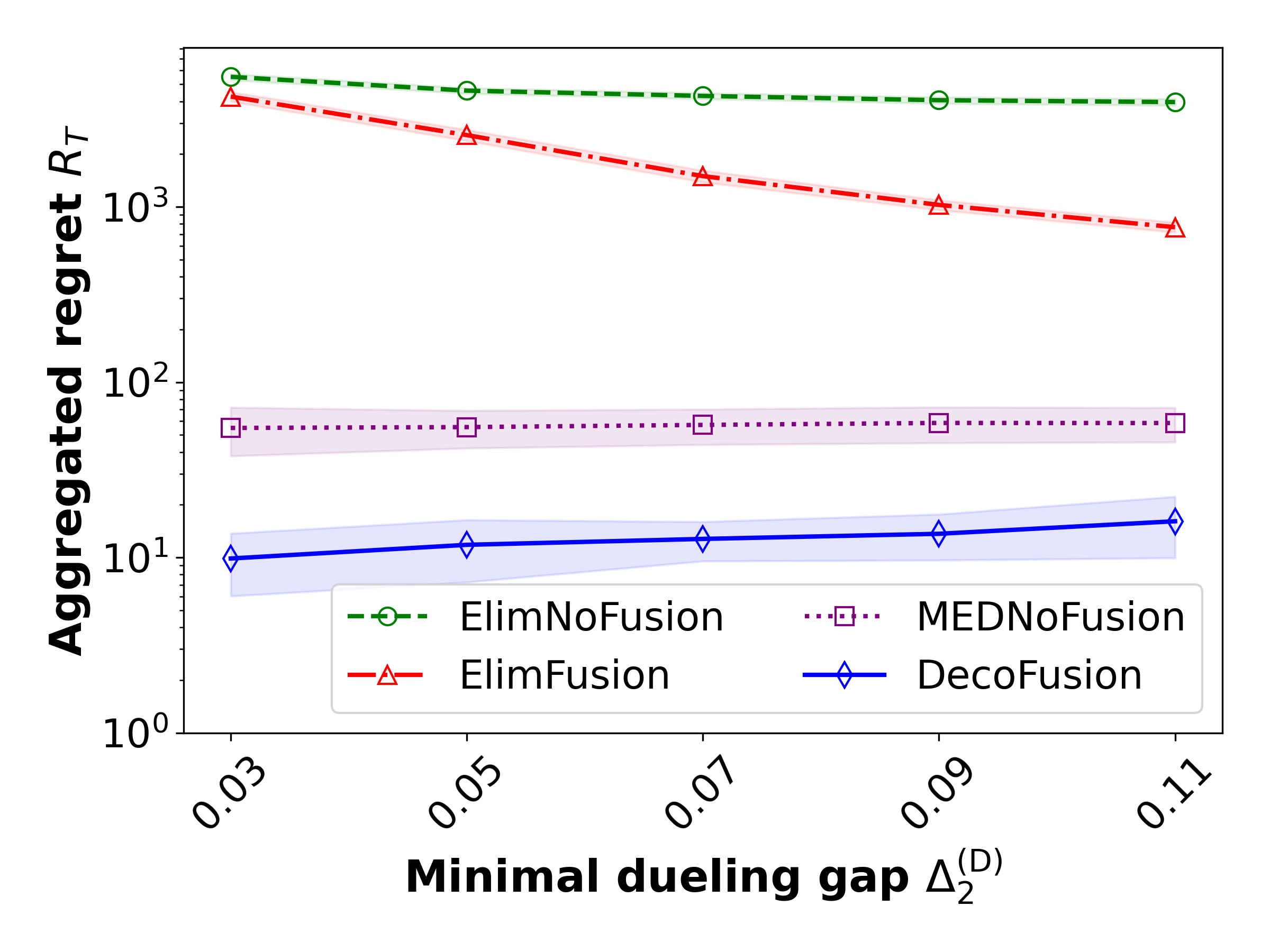}
        \label{fig:dueling-gap}
    }
    \hfill
    \subfigure[Vary \(\alpha\) in \defusion]{
        \includegraphics[width=0.48\columnwidth]{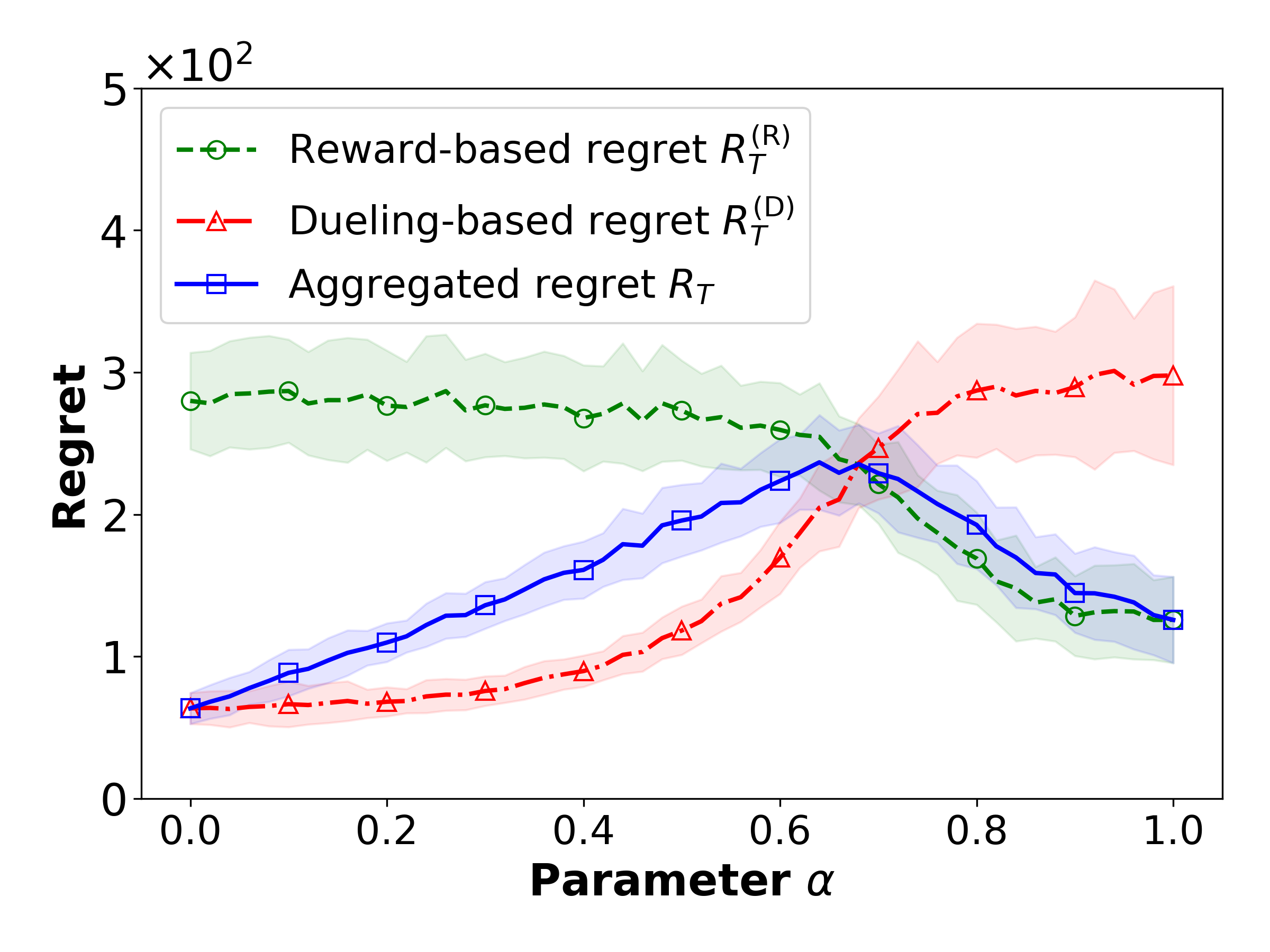}
        \label{fig:alpha-defusion}
    }
    \vspace{-0.1in}
    \caption{Regret comparison in different settings}
    \label{fig:experiments}
    \vspace{-0.2in}
\end{figure*}

The proof of Theorem~\ref{thm:defusion-regret} consists of two key claims (steps 2 and 3 of the proof in Appendix~\ref{app:proof-defusion-regret-bound}):
(i) in most time slots, the estimated optimal arms \(\hat k_t^{\rew}\) and \(\hat k_t^{\duel}\) from both feedback types are exactly the optimal arm \(1\), i.e., \(\sum_{t=1}^T \E[ \hat k_t^{\rew} = \hat k_t^{\duel} = 1 ] = T - O(1)\),
and (ii)
under the first claim, the regret of the randomized decision-making policy is upper bounded by the last regret term in~\eqref{eq:defusion-regret}.
While the claim (i) is proved via similar techniques in~\citet{honda2010asymptotically} and~\citet{komiyama2015regret},
the proof of claim (ii) needs a novel analysis to handle the regret costs of randomized decision-making.
It needs to construct a sampling threshold regarding both feedback and then bounds actual sampling times of reward feedback \(N_{k,T}\) for each arm and dueling feedback for each pair \(M_{k,\ell,T}\) when \(\ell = 1\) and \(\ell \neq 1\) case-by-case. The choice of the threshold \(\frac{\alpha^2}{\alpha^2 + (1 - \alpha)^2}\) is crucial for balancing the regret costs of both feedback types.

% compare with \defusion (Alg.~\ref{alg:defusion}) with \(\alpha\) threshold:
% \(O\left(\sum_{k\neq 1}\frac{\max\{\alpha\Delta_k^{\rew}, (1-\alpha) \Delta^{\duel}_k\} \log T}{\max\{(\Delta_k^{\rew})^2/\alpha, (\Delta^{\duel}_k)^2/(1-\alpha)\}} \right)\)

\textbf{Near-optimal regret when \(\ell_k^* = 1\).}
The simplified regret upper bound in~\eqref{eq:defusion-simplified-regret} tightly matches the simplified regret lower bound in~\eqref{eq:simplified-lower-bound} in terms of all non-trivial factors (except for a universal constant).
However, this simplified lower bound only holds when the most effective comparison arm \(\ell_k^*\) is the optimal arm \(1\) for all arms \(k\).
Without this condition, the regret upper bounds in Theorem~\ref{thm:defusion-regret} are worse than the general lower bound in~\eqref{eq:regret-lower-bound},
where the main gap comes from that \defusion often uses the dueling estimated optimal arm \(\hat k_t^{\duel}\) to duel with the exploration arm \(\exparm\), where the \(\hat k_t^{\duel}\) may not be the most effective comparison arm \(\ell_k^*\).
This issue also exists in the design of optimal dueling bandit algorithms~\citep[RMED1 and RMED2]{komiyama2015regret}, which is only partially resolved by assuming the knowledge of time horizon \(T\) in dueling bandit literature.

\textbf{Physical meanings for regret when \(\alpha = 0\) or \(\alpha = 1\).}
% We point out two special cases of the regret bound in~\eqref{eq:defusion-regret} when \(\alpha = 0\) and \(\alpha = 1\).
When either \(\alpha = 0\) or \(\alpha = 1\), the leading term of the regret bound in~\eqref{eq:defusion-regret} vanishes, and the regret becomes a \(T\)-independent constant.
Because when \(\alpha = 0\), the reward-based regret \(\rregret\) is not counted in the regret, i.e., reward feedback is \emph{free}! In this case, the threshold \(\frac{\alpha^2}{\alpha^2 + (1 - \alpha)^2} = 0\) in the randomized decision-making, and the algorithm always explores arms via reward feedback (i.e., \emph{free exploration}) and exploit the estimated optimal arm \(\hat k_t^{\rew}\) by the dueling feedback, which only incurs a constant regret.
Similarly when \(\alpha = 1\), the duel feedback is free, and the algorithm enjoys free exploration from dueling feedback
and achieves a constant regret as well.
% and exploit the estimated optimal arm \(\hat k_t^{\duel}\) by the reward feedback, which also only incurs a constant regret.
As illustrated in Figure~\ref{fig:alpha-impact-overview}, such an advantage is only achieved by \defusion, while \elimfusion always suffers \(O(\log T)\) regret cost.

% With the imbalanced decision making in Algorithm~\ref{alg:defusion}, we can further improve the regret upper bound as follows, \begin{align}
%     O\left( \sum_{k\neq 1}\frac{\max\{\Delta_k, \Delta^{\duel}_k\} \log T}{\max\{(\Delta_k^{\rew})^2/\alpha, (\Delta^{\duel}_k)^2/(1-\alpha)\}} \right),
% \end{align}
% which is slightly worse than the current lower bound which is \begin{align}
%     \Omega\left( \sum_{k\neq 1}\frac{\log T}{\max\{\Delta_k/\alpha, \Delta^{\duel}_k/(1-\alpha)\}} \right).
% \end{align}
% When the two maximums are achieved by the same feedback, the upper bound matches the lower bound up to a constant factor.
% \begin{itemize}
%     \item[\tdmark] Which can be further strengthened? lower or upper bound?
% \end{itemize}

%% file: sections/6_experiments.tex
% !TeX root = ../dueling-reward-bandits.tex
\section{Experiments}\label{sec:experiments}

This section reports the numerical experiments of the proposed fusion algorithms.
We compare our algorithms with two baselines:
1) \textsc{ElimNoFusion} that maintains two separate sets for arm elimination, one based on reward feedback and the other on dueling;
2) \textsc{MEDNoFusion} (Minimum Empirical Divergence) that deploys the optimal algorithms from \citet{honda2010asymptotically} and \citet{komiyama2015regret} to choose arm $k_t$ and dueling pair $(k_{1,t}, k_{2,t})$ {separately}.
% by setting the input parameter $\alpha$ of \defusion to either $0$ or $1$. 

Figures~\ref{fig:final-regret},~\ref{fig:reward-gap}, and~\ref{fig:dueling-gap} illustrate the trends of aggregated regret ($\alpha=0.5$) in various settings.
The setup details are in Appendix~\ref{app:experiment-setup}.
\defusion outperforms all other algorithms across all settings.
In Figure~\ref{fig:final-regret}, the aggregated regret of \defusion is around \(41\) times lower than that of \elimfusion.
While \elimfusion outperforms \textsc{ElimNoFusion} by \(81.3\%\), and \defusion outperforms \textsc{MEDNoFusion} by \(47.5\%\),  both elimination-based approaches are worse than the other two due to the suboptimality of the elimination mechanism.
% \todo{add the detailed improvement percentage: fusion vs nofusion, decofusion vs. elimfusion}
% Additionally, Figures~\ref{fig:reward-gap} and~\ref{fig:dueling-gap} investigate the impact of reward and dueling gaps on the algorithm performance. Increasing either gap makes it easier to differentiate suboptimal arms and thus saves the regret costs of the learning algorithm, and the improvement for two elimination-based algorithms is more evident than that of the optimal algorithms (\textsc{MEDNoFusion} and \defusion). The slight increase of the regret of \defusion in Figure~\ref{fig:dueling-gap} is because the regret cost in the warm-up phase (Line~\ref{line:defusion-warm-up}) --- increasing in terms of the dueling gap ---  
% take a large portion among the small total regret.
Figures~\ref{fig:reward-gap} and~\ref{fig:dueling-gap} examine the impact of reward and dueling gaps on total regret.
As either gap increases, distinguishing suboptimal arms becomes easier, thereby reducing the regret cost of the algorithms.
% This improvement is more evident for the two elimination-based algorithms than the optimal ones (\textsc{MEDNoFusion} and \defusion). 
The slight increase in the regret of \defusion in Figure~\ref{fig:dueling-gap} is due to the regret incurred during the warm-up (Line~\ref{line:defusion-warm-up}), which scales with the dueling gap and constitutes a large portion of the relatively small total regret.

% The results demonstrate that  \tdmark~revise the illustration for the first three figures.

Figure~\ref{fig:alpha-defusion} investigates the impact of the parameter \(\alpha\) on \defusion.
The aggregated regret (\textcolor{blue}{blue}) reaches its maximum when \(\alpha = 0.5\) and decreases as \(\alpha\) is close to either \(0\) or \(1\), which verifies its free exploration property and constant regrets when \(\alpha = 0\) or \(1\).
The curves of reward-based and dueling-based regrets (\textcolor{red}{red} and \textcolor{green!50!black}{green}) corroborate the effectiveness of the randomized decision-making policy: when \(\alpha=0\), the algorithm assigns all explorations to reward feedback, while the case of \(\alpha = 1\) assigns all explorations to dueling feedback,
and one can tune \(\alpha\) to balance the regret cost between the two types of feedback.

% compares the aggregated regret (\(\regret\)), as well as the reward-based (\(\rregret\)) and dueling-based (\(\dregret\)) regrets of \defusion for different \(\alpha\) values. The results indicate that \defusion effectively balances exploration between the two types of feedback under different values of \(\alpha\).

% We consider the following baselines:
% \begin{enumerate}
%     \item Elimination by reward alone (ignore dueling feedback)
%     \item Elimination by dueling alone (ignore reward feedback)
%     \item \citet{honda2010asymptotically}: \defusion with \(\alpha=0\) input, but \(\alpha = 0.5\) for regret \(\regret\) calculation (ignore dueling feedback)
%     \item \citet{komiyama2015regret}: \defusion with \(\alpha=1\) input, but \(\alpha = 0.5\) for regret \(\regret\) calculation (ignore reward feedback)
% \end{enumerate}

% \todo{consider use minipage to put figures~\ref{fig:final-regret},~\ref{fig:reward-regret}, and~\ref{fig:dueling-regret} together, and put~\ref{fig:alpha-defusion} in a separate minipage.}

% naively run two algorithms
% + Elimination + MED

% MED in different \(\alpha\), and plot three regrets: two individual regrets and the compound one.

%% file: sections/7_conclusion.tex
% !TeX root = ../dueling-reward-bandits.tex
\section{Conclusion}

This paper studies the fusion of reward and dueling feedback in stochastic multi-armed bandits, called \drbandit.
We derived regret lower bounds for \drbandit, and proposed two algorithms, \elimfusion and \defusion.
\elimfusion provides a simple and efficient way to fusion both feedback types via sharing the same candidate arm set, but its regret is suboptimal regarding a multiplicative factor of the number of arms.
\defusion, on the other hand, is designed to achieve the optimal regret up to a constant factor, by decoupling the suboptimal arms into two sets for reward and dueling feedback.
Both algorithms and the lower bound suggest that the advantage of fusing both feedback types is that it saves the higher regret costs among both feedback types for each suboptimal arm, achieving a better regret than solely relying on any one of the two.

% \section*{Impact Statement}

% This paper presents work whose goal is to advance the field of Machine Learning. There are many potential societal consequences of our work, none which we feel must be specifically highlighted here.

% If needed, discuss future works: \begin{itemize}
%     \item[\tdmark] mis-specification/
%           the orders of arms are different but correlated regarding reward and dueling feedback. from RL literatures.
%     \item[\tdmark]~how to combine offline and online datasets?
%     \item[\tdmark] treat the order perturbation as the attack from adversarial, and devise robust algorithms.
% \end{itemize}

%% file: sections/2_related_works.tex
% !TeX root = ../dueling-reward-bandits.tex
\section{Related Works}\label{sec:related-works}

% \xtimes~whether this question has been studied from the RL side or other related online learning literature.

% \tdmark~Acknowledge other similar works on related topics, like RLHF, and check Thodoris' Moderation paper.

% \check~related works to conversational bandits

\textbf{Stochastic bandits with absolute or relative feedback.}
Stochastic multi-armed bandits (\mab) is a fundamental online learning model, introduced by~\citet{lai1985asymptotically} and later comprehensively studied, e.g., by~\citet{bubeck2012regret,slivkins2019introduction,lattimore2020bandit}. Later on, various extensions of the stochastic bandits have been proposed, such as the linear bandits~\citep{abbasi2011improved}, contextual bandits~\citep{li2010contextual}, and combinatorial bandits~\citep{chen2013combinatorial}.
The relative feedback in the stochastic bandits setting is known as dueling bandits, which is initialized by~\citet{yue2012k} and studied by~\citet{ailon2014reducing,komiyama2015regret,sui2018advancements,saha2022versatile}, etc.
All of the above models assume that the learner receives either the absolute feedback, i.e., the reward of the selected arm,
or the relative feedback, i.e., the winning arm in a pair of arms, but not both,
which are different from the study on fusion of absolute and relative feedback in this paper.

% \textbf{Stochastic bandits with relative feedback}

% In the dueling bandits setting, the learner picks a pair of arms and receives the winning arm as the dueling (relative) feedback, while in our \drbandit, we consider the fusion of the absolute and relative feedback.

\textbf{Bandits with additional information.}
While we are the first to study the fusion of absolute and relative feedback in the stochastic bandits setting, there are prior works considering the absolute (reward) feedback with some types of ``augmentation'' feedback.
In the line of works of conversational bandits~\citep{zhang2020conversational,wang2023efficient,li2024fedconpe},
a learner interacts with a contextual linear bandit model, and
besides receiving the reward feedback from puling arms, the learner can also occasionally query ``key-terms'' to collect some side information of the bandit model.
Another line of works considers the bandits with ``hints'', where the ``hint'' can be additional reward observations~\citep{yun2018multi,lindstaahl2020predictive}, initial reward mean guesses~\citep{cutkosky2022leveraging}, or the order of the reward realizations among two or more arms~\citep{bhaskara2022online}, etc.
However, these types of augmentation feedback provide information directly related to the reward mean parameters (i.e., used to estimate the reward means),
which is different from the relative feedback depending on dueling probabilities in the dueling bandits setting, and therefore, they are different from our study on the fusion of absolute and relative feedback in this paper.

\textbf{Beyond stochastic bandits.} The stochastic bandit literature is only a small subset of the wide online learning research area~\citep{orabona2019modern}.
In contrast to the stochastic setting, the adversarial bandits---including adversarial \mab~\citep{auer2002nonstochastic} and adversarial dueling bandits~\citep{saha2021adversarial}---assume that the environment is determined by an adversary and may adapt to the learner's strategy.
Generalizing the bandits setting with state transition, reinforcement learning (RL)~\citep{sutton1998introduction} is another popular online learning model, where the learner interacts with the environment and receives reward feedback based on the state-action pairs.
Noticeably, there is a reinforcement learning with human feedback (RLHF) model~\citep{wang2023rlhf,ouyang2022training} in RL literature, which can be regarded as a counterpart of the dueling bandits in bandits literature, where the learner receives pairwise comparison (relative) feedback.
While there is vast online learning literature beyond the stochastic bandits, due ot the importance of the stochastic \mab, our study on the fusion of the absolute and relative feedback on bandits is a novel and unexplored topic.

%% file: sections/10_deferred_pseudo_code.tex
\section{Deferred Pseudo-code}\label{app:deferred_pseudo_code}

This section presents the pseudo-code of initial warm-up phase (Algorithm~\ref{alg:initial-phase}) and the statistics update phase (Algorithm~\ref{alg:statistics-update}) of the proposed \drbandit~algorithms.

\begin{algorithm}[htbp]
    \caption{Warm-up (Initial phase)}
    \label{alg:initial-phase}
    \begin{algorithmic}[1]
        % \STATE \textcolor{blue}{\textbf{Warm-up}}
        \STATE Duel each pair of arms \((k,\ell)\) for \(k\neq \ell\) once (in total \(K(K-1)\) times), meanwhile, query each arm for rewards in a round-robin manner
        \STATE \(t\gets K(K-1)\)
        \STATE Update \(M_{k,\ell,t}\) and \(\hat{\nu}_{k,\ell,t}\) for all pairs \((k,\ell)\in\mathcal{K}\times\mathcal{K}\) (\(k\neq \ell\)), and \(N_{k,t}\), \(\hat{\mu}_{k,t}\) for all arms \(k\in\mathcal{K}\)
    \end{algorithmic}
\end{algorithm}

\begin{algorithm}[htbp]
    \caption{Statistics update}
    \label{alg:statistics-update}
    \begin{algorithmic}[1]
        \FOR{all arm pairs \((k,\ell)\in \mathcal{K}\times \mathcal{K}\)}

        \IF{\((k,\ell) = (k_{t,1}, k_{t,2}) \text{ or } (k_{t,2}, k_{t,1})\)}
        \STATE \(M_{k,\ell,t+1} \gets M_{k,\ell,t} + 1\)
        \STATE \(\hat{\nu}_{k,\ell,t+1} \gets \frac{\hat{\nu}_{k,\ell,t}M_{k,\ell,t} + \1{Y_{k,\ell,t} = k}}{M_{k,\ell,t+1}}\)

        \ELSE
        \STATE \(M_{k,\ell,t+1} \gets M_{k,\ell,t}\) and \(\hat{\nu}_{k,\ell,t+1} \gets  \hat{\nu}_{k,\ell,t}\)

        \ENDIF
        \ENDFOR

        \FOR{all arms \(k\in \mathcal{K}\)}

        \IF{\(k = k_{t}\)}
        \STATE \(N_{k, t+1} \gets N_{k,t} + 1\)
        \STATE \(\hat{\mu}_{k,t+1} \gets \frac{\hat{\mu}_{k,t}N_{k,t} + X_{k_t,t}}{N_{k,t+1}}\)

        \ELSE
        \STATE \(N_{k, t+1} \gets N_{k, t}\) and \(\hat{\mu}_{k,t+1} \gets  \hat{\mu}_{k,t}\)

        \ENDIF

        \ENDFOR
    \end{algorithmic}
\end{algorithm}

%% file: sections/8_lower_bound_proof.tex
% !TeX root = ../dueling-reward-bandits.tex
\section{Proof of Lower Bound}\label{app:proof-lower-bound}

\begin{proof}[Proof of Lemma~\ref{lma:information-lower-bound}]
    \textbf{Step 1. Instance and event construction.}
    Pick any suboptimal arm \(k\neq 1\).
    We use the standard parameters defined in model section as the original instance \(\mathcal{I}\) and then construct an alternative instances \(\mathcal{I}'\) with parameters with a prime, e.g., \(\mu_k'\) and \(\nu_{k,\ell}'\), as follows,
    \begin{itemize}
        \item Reward means \(\mu_\ell' = \begin{cases}
                  \mu_\ell         & \text{ if } \ell\neq k
                  \\
                  \mu_1 + \epsilon & \text{ if } \ell = k
              \end{cases},\)
        \item Dueling probabilities \(
              \nu_{\ell_1, \ell_2} = \begin{cases}
                  \frac 1 2 + \epsilon_{\ell_2} & \text{ if } \ell_1 = k \text{ and } \ell_2 < k
                  \\
                  \frac 1 2 - \epsilon_{\ell_1} & \text{ if } \ell_1 <  k \text{ and } \ell_2 = k
                  \\
                  \nu_{\ell_1, \ell_2}          & \text{otherwise}
              \end{cases},
              \)
    \end{itemize}
    where the \(\epsilon>0\) is a small constant to be determined later, and
    the \(\epsilon_\ell\) parameters are chosen such that \(\kl(\nu_{k,\ell}, \frac{1}{2} + \epsilon_\ell) = \kl(\nu_{k,\ell}, \frac{1}{2}) + \epsilon\).
    Under this instance construction, the optimal arms are arm \(1\) and arm \(k\) in the original and alternative instances, respectively.
    All reward distributions are Bernoulli.
    We denote \(\mathbb{E},\mathbb{P}\) and \(\mathbb{E}',\mathbb{P}'\) as the expectation and probability under the original and alternative instances, respectively.

    To facilitate the rest of the analysis, we define the empirical KL-divergence as follows,
    \begin{align}
        \widehat{\KL}^{\duel}_\ell (n)
         & \coloneqq \sum_{s=1}^n \log\left( \frac{Y_{k,\ell, (s)}\nu_{k,\ell} + (1 - Y_{k,\ell,(s)})(1 - \nu_{k,\ell})}{Y_{k,\ell, (sn)}\nu_{k,\ell}' + (1 - Y_{k,\ell,(s)})(1 - \nu_{k,\ell}')} \right),
        \\
        \widehat{\KL}^{\rew}_k (n)
         & \coloneqq \sum_{s=1}^n \log\left( \frac{X_{k,(s)}\mu_k + (1-X_{k,(s)})(1-\mu_k)}{X_{k,(s)}\mu_k' + (1-X_{k,(s)})(1-\mu_k')} \right),
        \\
        \widehat{\KL}(\mathcal{I}, \mathcal{I}')
         & \coloneqq \widehat{\KL}^{\rew}_k (N_{k,T}) + \sum_{\ell < k} \widehat{\KL}^{\duel}_\ell (M_{k,\ell,T}),
    \end{align}
    where the subscript \((s)\) refers to the \(s^{\text{th}}\) observation of the corresponding random variable, which differs from the time index.
    With these empirical KL-divergences defined, for any event \(\mathcal{E}\), we have the following relation holds (change of measure),  \begin{align}
        \mathbb{P}'(\mathcal{E}) = \mathbb{E}\left[\1{\mathcal{E}}\exp(-\widehat{\KL}(\mathcal{I}, \mathcal{I}'))\right].
    \end{align}

    In the end of the first step, we define two events as follows,
    \begin{align}
        \mathcal{D}_1 & \coloneqq \left\{ \sum_{\ell < k} \kl\left(\nu_{k,\ell}, \nu_{k,\ell}'\right) M_{k,\ell, T}
        +
        \kl(\mu_k, \mu_k')N_{k,T} <  (1-\epsilon) \log T \right\},
        \\
        \mathcal{D}_2 & \coloneqq \left\{ \widehat{\KL}(\mathcal{I}, \mathcal{I}')\le \left( 1-\frac{\epsilon}{2} \right) \log T\right\}.
    \end{align}
    In the remaining of this proof, we use two steps to prove that \(\mathbb{P}(\mathcal{D}_1) = o(1)\), which then implies the lemma.

    \textbf{Step 2. Prove \(\mathbb{P}(\mathcal{D}_1\cap \mathcal{D}_2) = o(1)\).} We first apply the change of measure argument to transfer the measure of the event \(\mathcal{D}_1\cap \mathcal{D}_2\) from instance \(\mathcal{I}\) to instance \(\mathcal{I}'\),
    \begin{align}
        \mathbb{P}'(\mathcal{D}_1\cap \mathcal{D}_2)
         & = \mathbb{E}[\1{\mathcal{D}_1\cap \mathcal{D}_2}\exp(-\widehat{\KL}(\mathcal{I}, \mathcal{I}'))]
        \ge \mathbb{E}[\1{\mathcal{D}_1\cap \mathcal{D}_2}] T^{-(1-\frac{\epsilon}{2})},
    \end{align}
    which yields
    \begin{align}
        \mathbb{P}(\mathcal{D}_1\cap \mathcal{D}_2)
         & \le T^{1-\frac{\epsilon}{2}} \mathbb{P}'(\mathcal{D}_1\cap \mathcal{D}_2)
        \\
         & \le T^{1 - \frac{\epsilon}{2}} \mathbb{P}'\left(N_{k,T} <  \frac{(1-\epsilon)\log T}{\kl(\mu_k, \mu_k')}  \right)
        \\
         & = T^{1 - \frac{\epsilon}{2}} \mathbb{P}'\left( T - N_{k,T} > T -  \frac{(1-\epsilon)\log T}{\kl(\mu_k, \mu_k')}  \right)
        \\
         & \le T^{1 - \frac{\epsilon}{2}} \frac{T - \mathbb{E}'[N_{k,T}]}{T - ({(1-\epsilon)\log T}/{\kl(\mu_k, \mu_k')})} \label{eq:markov-inequality-for-D-1-and-D-2}
        \\
         & =  T^{1 - \frac{\epsilon}{2}} \frac{ \sum_{\ell\neq k}\mathbb{E}'[N_{\ell,T}]}{T - ({(1-\epsilon)\log T}/{\kl(\mu_k, \mu_k')})}
        \\
         & \le o(T^{\gamma - \frac{\epsilon}{2}}) = o(1), \label{eq:apply-consistent-policy}
    \end{align}
    where inequality~\eqref{eq:markov-inequality-for-D-1-and-D-2} is due to the Markov inequality,
    and inequality~\eqref{eq:apply-consistent-policy} is due to the consistent policy assumption (Definition~\ref{def:consistent-algorithm}) that for any suboptimal arm \(\ell\neq k\) under instance \(\mathcal{I}'\), we have \(\mathbb{E}'[N_{\ell,T}] = o(T^\gamma)\) for any positive \(\gamma\).

    \textbf{Step 3. Prove \(\mathbb{P}(\mathcal{D}_1\setminus \mathcal{D}_2) = o(1)\).} To prove this claim, we first telescope the \(\mathbb{P}(\mathcal{D}_1\setminus \mathcal{D}_2)\) and focus on bounding of the empirical KL-divergence as follows,
    \begin{align}
         & \quad\,\, \mathbb{P}(\mathcal{D}_1\setminus \mathcal{D}_2)
        \\
         & = \mathbb{P} \left( \sum_{\ell < k} \kl\left(\nu_{k,\ell}, \nu_{k,\ell}'\right) M_{k,\ell, T}
        +
        \kl(\mu_k, \mu_k')N_{k,T} <  (1-\epsilon) \log T,
        \, \widehat{\KL}(\mathcal{I}, \mathcal{I}') > \left( 1-\frac{\epsilon}{2} \right) \log T \right)
        \\
         & \le \mathbb{P} \left( \sum_{\ell < k} \kl\left(\nu_{k,\ell}, \nu_{k,\ell}'\right) M_{k,\ell, T}
        +
        \kl(\mu_k, \mu_k')N_{k,T} <  (1-\epsilon) \log T, \right.
        \\
         & \qquad\qquad\qquad\left.
        \max_{\substack{
                n_k, m_\ell \in \mathbb{N}^+,\forall \ell < k:  \sum_{\ell < k} \kl\left(\nu_{k,\ell}, \nu_{k,\ell}'\right) m_\ell
        \\
                + \kl(\mu_k, \mu_k') n_k < (1-\epsilon) \log T
            }}
        \widehat{\KL}^{\rew}_k (n_k) + \sum_{\ell < k} \widehat{\KL}^{\duel}_\ell (m_\ell) > \left( 1-\frac{\epsilon}{2} \right) \log T \right)
        \\
         & \le \mathbb{P} \left(
        \max_{\substack{
                n_k, m_\ell \in \mathbb{N}^+,\forall \ell < k: \sum_{\ell < k} \kl\left(\nu_{k,\ell}, \nu_{k,\ell}'\right) m_\ell
        \\
                + \kl(\mu_k, \mu_k') n_k < (1-\epsilon) \log T
            }}
        \widehat{\KL}^{\rew}_k (n_k) + \sum_{\ell < k} \widehat{\KL}^{\duel}_\ell (m_\ell) > \left( 1-\frac{\epsilon}{2} \right) \log T \right).
    \end{align}

    Next, we a variant of the maximal law of large numbers (LLN)~\citep[Lemma 10.5]{bubeck2010bandits} to bound the empirical KL-divergence.
    To guarantee that all sample times \(N_{k,T}\) and \(M_{k,\ell, T}\) are large enough to apply LLN, we set the lower bound of \(N_{k,T}\) and \(M_{k,\ell, T}\) as \(\delta\log T\) for some small constant \(\delta > 0\) as follows,
    \begin{align}
         & \quad\,\, \frac{\max_{\substack{
                    n_k, m_\ell \in \mathbb{N}^+,\forall \ell < k: \sum_{\ell < k} \kl\left(\nu_{k,\ell}, \nu_{k,\ell}'\right) m_\ell
        \\
                    + \kl(\mu_k, \mu_k') n_k < (1-\epsilon) \log T
                }}
            \widehat{\KL}^{\rew}_k (n_k) + \sum_{\ell < k} \widehat{\KL}^{\duel}_\ell (m_\ell)}{\log T}
        \\\label{eq:kl-average-limit-range}
         & \le \frac{\max_{\substack{
                    n_k, m_\ell \in \mathbb{N}^+, n_k, m_\ell > \delta\log T, \forall \ell < k, :
        \\
                    \sum_{\ell < k} \kl\left(\nu_{k,\ell}, \nu_{k,\ell}'\right) m_\ell + \kl(\mu_k, \mu_k') n_k < (1-\epsilon) \log T
                }}
            \widehat{\KL}^{\rew}_k (n_k) + \sum_{\ell < k} \widehat{\KL}^{\duel}_\ell (m_\ell)}{\log T}
        \\
         &
        \qquad\qquad\qquad\qquad\qquad\qquad\qquad\qquad\qquad\qquad\qquad\qquad
        + \frac{\delta (k-1) }{\min_{\ell < k} \kl\left(\nu_{k,\ell}, \nu_{k,\ell}'\right)} + \frac{\delta}{\kl(\mu_k, \mu_k')}.
    \end{align}

    Then, because the maximal LLM~\citep[Lemma 10.5]{bubeck2010bandits} implies \(\lim_{N\to\infty} \max_{1\le n \le N}\frac{\widehat{\KL}_k^{\rew}(n)}{N} = \kl(\mu_k, \mu_k')\) and  \(\lim_{M\to\infty} \max_{1\le m \le M}\frac{\widehat{\KL}_\ell^{\duel}(m)}{M} = \kl(\nu_{k,\ell}, \nu_{k,\ell}')\) for any arm \(\ell < k\), almost surely (a.s.), we bound the first term in the right-hand side of~\eqref{eq:kl-average-limit-range} by the maximal LLN as follows,
    \begin{align}\label{eq:apply-lemma-10.5}
         & \limsup_{T\to \infty} \frac{\max_{\substack{
                    n_k, m_\ell \in \mathbb{N}^+, n_k, m_\ell > \delta\log T, \forall \ell < k, :
        \\
                    \sum_{\ell < k} \kl\left(\nu_{k,\ell}, \nu_{k,\ell}'\right) m_\ell + \kl(\mu_k, \mu_k') n_k < (1-\epsilon) \log T
                }}
            \widehat{\KL}^{\rew}_k (n_k) + \sum_{\ell < k} \widehat{\KL}^{\duel}_\ell (m_\ell)}{\log T}
        \le 1-\epsilon.
    \end{align}

    Therefore, combining~\eqref{eq:kl-average-limit-range} and~\eqref{eq:apply-lemma-10.5}, we have \begin{align}
         & \frac{\max_{\substack{
                    n_k, m_\ell \in \mathbb{N}^+,\forall \ell < k: \sum_{\ell < k} \kl\left(\nu_{k,\ell}, \nu_{k,\ell}'\right) m_\ell
        \\
                    + \kl(\mu_k, \mu_k') n_k < (1-\epsilon) \log T
                }}
            \widehat{\KL}^{\rew}_k (n_k) + \sum_{\ell < k} \widehat{\KL}^{\duel}_\ell (m_\ell)}{\log T}
        \le 1 - \epsilon + \Theta(\delta), \text{ a.s.}
    \end{align}
    Noticing \(1-\epsilon < 1 - \frac{\epsilon}{2}\) and letting \(\delta \to 0\), we proved that
    \begin{align}
        \mathbb{P}(\mathcal{D}_1\setminus \mathcal{D}_2)
         & \le \mathbb{P} \left(
        \max_{\substack{
                n_k, m_\ell \in \mathbb{N}^+,\forall \ell < k: \sum_{\ell < k} \kl\left(\nu_{k,\ell}, \nu_{k,\ell}'\right) m_\ell
        \\
                + \kl(\mu_k, \mu_k') n_k < (1-\epsilon) \log T
            }}
        \widehat{\KL}^{\rew}_k (n_k) + \sum_{\ell < k} \widehat{\KL}^{\duel}_\ell (m_\ell) > \left( 1-\frac{\epsilon}{2} \right) \log T \right)
        \\
         & = o(1).
    \end{align}

    Lastly, by letting \(\epsilon\) be infinitesimally small, we have \(\mathbb{P}(\mathcal{D}_1) = o(1)\), which implies the lemma.
\end{proof}

\begin{proof}[Proof of Theorem~\ref{thm:regret-lower-bound} (regret lower bound)]
    We first present the regret decomposition in terms of the reward and dueling feedback as follows,
    \begin{align}\label{eq:regret-decomposition}
        R_T = \sum_{k>1} \left( \alpha \Delta_k^{\rew} N_{k,T} + \sum_{\ell < k} (1-\alpha)\left( \Delta^{\duel}_k + \Delta^{\duel}_\ell \right) M_{k,\ell, T}\right).
    \end{align}

    Then, with Lemma~\ref{lma:information-lower-bound}, its corresponding regret of each suboptimal arm \(k \neq 1\) can be lower bounded by the following optimizing problem,
    \begin{align}
        \min \quad
         & \alpha \Delta_k^{\rew} N_{k,T} + \sum_{\ell < k} (1-\alpha)\left( \Delta^{\duel}_k + \Delta^{\duel}_\ell \right) M_{k,\ell, T} \\
        \text{s.t.} \quad
         & \sum_{\ell < k} \kl\left(\nu_{k,\ell}, \frac{1}{2}\right) M_{k,\ell, T} + \kl(\mu_k, \mu_1) N_{k,T} \ge (1-o(1)) \log T,
    \end{align}
    whose asymptotical solution is as follows,
    \begin{align}
         & \quad \liminf_{T\to\infty} \frac{\E[\Delta_k^{\rew} N_{k,T} + \sum_{\ell < k} (1-\alpha)\left( \Delta^{\duel}_k + \Delta^{\duel}_\ell \right) M_{k,\ell, T}]}{\log T}
        \ge \min\left\{  \frac{\alpha\Delta_k}{\kl(\mu_k, \mu_1)}
        , \min_{\ell < k} \frac{(1-\alpha) (\Delta^{\duel}_k + \Delta^{\duel}_{\ell})}{\kl\left( \nu_{k,\ell}, \frac 1 2 \right)}\right\}.
    \end{align}
    Lastly, we substitute the above lower bound for each terms in~\eqref{eq:regret-decomposition} and obtain the regret lower bound.
\end{proof}

%% file: sections/9_upper_bound_proof.tex
% !TeX root = ../dueling-reward-bandits.tex
\section{Proof of Upper Bounds}

\subsection{Regret Upper Bound for \elimfusion (Algorithm~\ref{alg:elimination-fusion}, Theorem~\ref{thm:elimination-fusion})}\label{app:proof-elimination-fusion}

\begin{proof}[Proof of Theorem~\ref{thm:elimination-fusion}]
    \textbf{Step 1. Bound small probability event.} We first show the concentration of the reward and dueling probabilities as follows,
    \begin{align}
        \P\left( \abs*{\hat{\mu}_{k,t} - \mu_k} \ge \sqrt{\frac{2\log (Kt/\delta)}{N_{k,t}}}  \right)
         & = \sum_{n=1}^t \P\left( \left.\abs*{\hat{\mu}_{k,t} - \mu_k} \ge \sqrt{\frac{2\log (Kt/\delta)}{n}} \right\rvert N_{k,t} = n \right)\P\left( N_{k,t} = n \right) \label{eq:total-probability}
        \\
         & \le \sum_{n=1}^t \P\left( \abs*{\hat{\mu}_{k,t} - \mu_k} \ge \sqrt{\frac{2\log (Kt/\delta)}{n}} \right)
        \\
         & \le \sum_{n=1}^t 2\exp\left( -4\log (Kt/\delta) \right) \label{eq:hoeffding}
        \\
         & \le \frac{2\delta^4}{K^4t^3}
    \end{align}
    where inequality~\eqref{eq:total-probability} follows from the formula of total probability, and
    inequality~\eqref{eq:hoeffding} follows from Hoeffding's inequality.

    Therefore, we have
    \begin{align}\label{eq:small-probability-event-reward}
        \P\left( \forall t\in\mathcal{T}, k\in\mathcal{K}, \abs*{\hat{\mu}_{k,t} - \mu_k} \ge \sqrt{\frac{2\log (Kt/\delta)}{N_{k,t}}}  \right)
         & \overlabel{a}\le \sum_{t=1}^T\sum_{k=1}^K \frac{\delta^4}{K^4t^3}
        \le \frac{\delta^4}{K^3} \sum_{t=1}^T \frac{1}{t^3} \le \frac{3\delta^4}{2K^3},
    \end{align}
    where inequality (a) follows from the union bound and the above concentration.

    With a similar argument, we can show that the dueling probabilities also concentrate as follows,
    \[
        \P\left( \forall t\in\mathcal{T}, k,\ell\in\mathcal{K}, \abs*{\hat{\nu}_{k,\ell,t} - \nu_{k,\ell}} \ge \sqrt{\frac{2\log (Kt/\delta)}{M_{k,\ell,t}}}  \right)
        \le \frac{3\delta^4}{4K^2},
    \]
    where the RHS's dominator different from~\eqref{eq:small-probability-event-reward} is because the number of arm pairs is \(K(K-1)/2\).

    \textbf{Step 2. Bound the maximal sample times for both types of feedback.}

    \emph{Elimination with reward feedback.}
    One suboptimal arm \(k\) is eliminated by reward feedback if the following event happens, \[
        \hat{\mu}_{1,t} - \hat{\mu}_{k,t} \ge
        \mu_1 - \mu_k - 2\sqrt{\frac{2\log (Kt/\delta)}{N_{k,t}}}
        > 2\sqrt{\frac{2\log (Kt/\delta)}{N_{k,t}}}
        ,\]
    which would happen on or before the rearranged expression as follows,
    \[
        N_{k,t} > \frac{32\log (Kt/\delta)}{(\Delta_k^{\rew})^2}.
    \]
    In other words, from the reward feedback and its corresponding elimination, the sample times of arm \(k\) are upper bounded by \(\frac{32\log (KT/\delta)}{(\Delta_k^{\rew})^2}\).

    \emph{Elimination with dueling feedback.} The sample times for an arm \(k\) from the dueling feedback is given by noticing that for any arm \(\ell < k\), when \[
        \hat{\nu}_{k,\ell,t} + \sqrt{\frac{2\log (Kt/\delta)}{M_{k,\ell,t}}} \le {\nu}_{k,\ell} + 2\sqrt{\frac{2\log (Kt/\delta)}{M_{k,\ell,t}}}
        <  \frac{1}{2}
        ,\]
    that is,
    \[
        M_{k,\ell,t} \ge \frac{8\log (Kt/\delta)}{(\Delta^{\duel}_k)^2},
    \]
    for all arms \(\ell\), then the arm \(k\) must have been eliminated from the candidate arm set \(\mathcal{C}\) by Line~\ref{line:dual-elimination} of Algorithm~\ref{alg:elimination-fusion}.
    Therefore, the sampling times of arm pair \((k,\ell)\) is upper bounded by \(\frac{8\log (KT/\delta)}{(\Delta^{\duel}_k)^2}\).
    % Then, the sample times of arm \(k\) by summing of all possible comparison arms \(\ell\in\mathcal{K} \setminus\{k\}\), that is, \(\frac{8(K-1)\log(KT/\delta)}{(\Delta^{\duel}_k)^2}\), we can further reduce the sample times by noticing that the arms are eliminated one-by-one, and possibly following the order of the arms in decreasing order of \(\mu_k\).
    % \tdmark~add the argument of considering the worst case of arm elimination order. \tdmark~assume the arm order are the same as their index.
    % \tdmark~think whether the assumption that \(\mu_k > \mu_\ell\) implies \(\nu_{k,k^*} > \nu_{\ell,k^*}\) is necessary.

    % \tdmark~Below we illustrate that while the suboptimal arms may be eliminated in any order due to stochasticity, eliminating arms in the normal order (i.e., from arm \(K\) to arm \(2\)) induces the highest regret.

    % Denote the arm elimination order as \(k_K, k_{K-1}, \ldots, k_{2}\), a permutation of the arm set \(\{2,3,\dots,K\}\).

    \textbf{Step 3. Bound the regret.}
    Each suboptimal arm \(k\) may be eliminated by either reward or dueling feedback. So, the regret incurred by the reward querying of arm \(k\) is upper bounded as follows,
    \begin{align}
        R_{k,T}^{\rew}
         & \le
        \Delta_k^{\rew} \cdot
        \min\left\{
        \frac{32\log (KT/\delta)}{(\Delta_k^{\rew})^2},
        \frac{K-1}{2} \cdot \frac{8\log(KT/\delta)}{(\Delta^{\duel}_k)^2}
        \right\}
        = \frac{4\Delta_k^{\rew}\log (KT/\delta)}{\max\{
            (\Delta_k^{\rew})^2/8, (\Delta_k^{\duel})^2/4(K-1)\}},
    \end{align}
    where the factor \(\frac{K-1}{2}\) is an upper bound of the ratio of sampling collection rate between reward and dueling feedback
    because the number of arms in the candidate arm set \(\mathcal{C}\) is at most \(\frac{K-1}{2}\) times smaller than the number of arm pairs in the set.

    Next, we bound the regret incurred by the dueling querying of arm \(k\).  If the arm \(k\) is eliminated by the dueling feedback, then the dueling regret due to arm \(k\) can be upper bounded as follows,
    \[
        \sum_{\ell<k} \frac{(\Delta^{\duel}_{\ell} + \Delta^{\duel}_k - 1)}{2} \cdot M_{\ell,k,T}
        \le
        \sum_{\ell< k} \Delta^{\duel}_k \frac{8\log (KT/\delta)}{(\Delta^{\duel}_k)^2}
        \le \frac{8(k-1)\log (KT/\delta)}{\Delta^{\duel}_k}.
    \]

    If the arm \(k\) is eliminated by the reward feedback, implying the algorithm has queried the reward of arm \(k\) for at most \(\frac{32\log(KT/\delta)}{(\Delta_k^{\rew})^2}\) times. During this period, dueling regret cost of any arm pair \((k,\ell)\) from some arm \(\ell < k\) is upper bounded as follows,
    \[
        \frac{(\Delta^{\duel}_{\ell} + \Delta^{\duel}_k - 1)}{2} \times 2\times \frac{32\log(KT/\delta)}{(\Delta_k^{\rew})^2}
        =  \frac{64\Delta_k^{\duel}\log(KT/\delta)}{(\Delta_k^{\rew})^2},
    \]
    where the factor \(2\) every reward query of arm \(k\) is accompanied by at most two dueling comparison of arm pairs \((k,\ell)\) from some arm \(\ell < k\).

    Taking the minimal among the above two cases, we have the dueling regret upper bound as follows,
    \begin{align}
        R_{k,T}^{\duel} & \le \Delta_k^{\duel}\cdot \min\left\{
        \frac{64\log(KT/\delta)}{(\Delta_k^{\rew})^2},
        \frac{8(k-1)\log (KT/\delta)}{(\Delta^{\duel}_k)^2}
        \right\}
        =  \frac{8\Delta^{\duel}_k\log (KT/\delta)}{\max\{(\Delta_k^{\rew})^2/8, (\Delta^{\duel}_k)^2/(k-1)\}}.
    \end{align}

    Lastly, we bound the total regret as follows,
    \begin{align}
        \regret
         & \le \sum_{k\neq 1} \alpha R_{k,T}^{\rew} + (1-\alpha) R_{k,T}^{\duel}
        \\
         & \le \sum_{k\neq 1} \frac{4\alpha\Delta_k^{\rew} \log (KT/\delta)}{\max\{(\Delta_k^{\rew})^2/8, (\Delta^{\duel}_k)^2/(K-1)\}}
        + \frac{8(1 - \alpha)\Delta^{\duel}_k\log (KT/\delta)}{\max\{(\Delta_k^{\rew})^2/8, (\Delta^{\duel}_k)^2/(k-1)\}}
        \\
         & \le \sum_{k\neq 1} \frac{8(\alpha\Delta^{\rew}_k + (1-\alpha)\Delta_k^{\duel}) \log (KT/\delta)}{\max\{(\Delta_k^{\rew})^2/8, (\Delta^{\duel}_k)^2/(K-1)\}}.
    \end{align}
\end{proof}

\subsection{Regret Upper Bound for \defusion (Algorithm~\ref{alg:defusion}, Theorem~\ref{thm:defusion-regret})}\label{app:proof-defusion-regret-bound}

\begin{proof}[Proof of Regret Upper Bound of Algorithm~\ref{alg:defusion}]
    \textbf{Step 1. Event definition.}
    We start the proof by defining the following events,
    \begin{align}
        \mathcal{A}_t
         & \coloneqq \left\{
        \hat\mu_{1, t} \ge \hat\mu_{k,t}, \forall k\in \hat{\mathcal{K}}^{\rew}_t, \hat{\mu}_{1,t} \ge \mu_1 - \delta,
        \text{ and } \hat\nu_{1,k,t} \ge \frac{1}{2}, \forall k\in \hat{\mathcal{K}}^{\duel}_t
        \right\},
        \\
        \mathcal{G}_t
         & \coloneqq \left\{
        \hat\mu_{k_t^{\rew},t} \le
        \max_{k\in\hat{\mathcal{K}}^{\rew}_t\setminus \{1\}} \mu_k + \delta
        \right\},
        \\
        \mathcal{H}_t
         & \coloneqq \bigcup_{k\in\hat{\mathcal{K}}^{\rew}_t} \left\{
        \hat{\mu}_{k_t^{\rew}, t} = \hat{\mu}_{k,t}
        \text{ and } \abs{\hat{\mu}_{k,t} - \mu_k} \ge \delta
        \right\},
        \\
        \mathcal{U}_t
         & \coloneqq \bigcup_{\mathcal{S}\in 2^{\hat{\mathcal{K}}^{\duel}_t} \setminus \{\emptyset\}}
        \left\{
        \hat\nu_{1, k, t} \ge \frac{1}{2}, \forall k\in\mathcal{S}
        \text{ and } \hat\nu_{1, k, t} < \frac{1}{2},\forall k\in\hat{\mathcal{K}}_t^{\duel}\setminus \mathcal{S}
        \right\},
    \end{align}
    where the parameter \(\delta\) is a small positive constant that satisfies \(0 < \delta < \Delta_2^{\rew}\).
    The \(\mathcal{A}_t\) refers to a good event, implying that the algorithm has correctly estimated the optimal arm at time slot \(t\), i.e., \(\hat k_t^{\rew} = \hat k_t^{\duel} = 1\),
    and the other three refer to the bad events that the algorithm may not correctly estimate the optimal arm, either due to bad reward estimates or bad dueling probability estimates.

    Notice that (1) if both events \(\mathcal{G}_t\) and \(\mathcal{H}_t\) do not happen, i.e., \(\mathcal{G}_t^{\text{C}}\cap \mathcal{H}_t^{\text{C}}\), then the first condition of event \(\mathcal{A}_t\) holds,
    and (2) if event \(\mathcal{U}_t\) does not happen, then the second condition of event \(\mathcal{A}_t\) holds.
    Putting these together, we have the following relation between these events,
    \begin{align}\label{eq:defusion-proof-event-relation}
        \mathcal{G}_t^{\text{C}} \cap \mathcal{H}_t^{\text{C}} \cap \mathcal{U}_t^{\text{C}} \subseteq \mathcal{A}_t,
        \text{ that is, }
        \mathcal{A}_t^{\text{C}} \subseteq
        \mathcal{G}_t \cup \mathcal{H}_t \cup \mathcal{U}_t.
    \end{align}

    In the rest of this proof, we show two claims:
    (1) the number of times that any of these three bad events \(\mathcal{G}_t\), \(\mathcal{H}_t\), and \(\mathcal{U}_t\) happens is bounded by \(O(1)\) (a term independent of time horizon \(T\)),
    and
    (2) for the time slots that the good event \(\mathcal{A}_t\) happens,
    the regret of \defusion is upper bounded as~\eqref{eq:defusion-regret} shows.

    \textbf{Step 2. Bound the number of happening times of the bad events.}
    Following~\citet[Lemmas 16 and 17]{honda2010asymptotically}, we bound \(\E\left[ \sum_{t=1}^T \1{\mathcal{G}_t} \right] \le O(1),
    \)
    and \(\E\left[ \sum_{t=1}^T \1{\mathcal{H}_t} \right] \le O(1)\).
    Following~\citet[Lemma 5]{komiyama2015regret},
    we bound \(\E\left[ \sum_{t=1}^T \1{\mathcal{E}_t} \right] \le O(e^{K - f(K)}).\)
    Putting these results together, we have \begin{align}
        \E\left[ \sum_{t=1}^T \1{\mathcal{A}_t^C} \right]
         & \le \E\left[ \sum_{t=1}^T \1{\mathcal{G}_t} \right] + \E\left[ \sum_{t=1}^T \1{\mathcal{H}_t} \right] + \E\left[ \sum_{t=1}^T \1{\mathcal{E}_t} \right]
        \\
         & \le O(1) + O(1) + O(e^{cK - f(K)}) \le O(e^{cK - f(K)}),
    \end{align}
    for some constant \(c > 0\).

    \textbf{Step 3. Bound regret.}
    Denote \(\beta \coloneqq \frac{\alpha^2}{\alpha^2 + (1 - \alpha)^2}\) as the threshold in Line~\ref{line:defusion-random-decision-choice} of Algorithm~\ref{alg:defusion}.
    We first define the following three quantities as the sufficient number of times for exploring arm \(k\) from different types of feedback, for sufficiently small \(\epsilon > 0\),
    \begin{align}
        N_{k}^{\suff}
         & \coloneqq \frac{(1 + \epsilon)\log T + f(K)}{\kl(\mu_k, \mu_1)},
        \\
        M_{k,1}^{\suff}
         & \coloneqq \frac{(1 + \epsilon)\log T + f(K)}{\kl(\nu_{k,1}, \frac{1}{2})},
        \\
        L_{k}^{\suff}
         & \coloneqq \min \left\{  N_{k,t}^{\suff} / (1-\beta),  M_{k,1,t}^{\suff} / \beta \right\}
        = \frac{(1 + \epsilon)\log T + f(K)}{\max\left\{(1-\beta)\kl(\mu_k, \mu_1), \beta \kl(\nu_{k,1}, \frac{1}{2})\right\}}.
    \end{align}

    Next, we bound the expected number of times of exploring arm \(k\) for different types of feedback as follows,
    \begin{align}
        \E[N_{k,T}]
         & \le K/2 + \E\left[ \sum_{t=1}^T \1{\mathcal{A}_t^C} \right] +  (1 - \beta) \cdot \E\left[ \sum_{t=1}^T \1{\mathcal{J}_{k,t} \text{ and } \mathcal{A}_t} \right], \label{eq:bound-N}
        \\
        \E[M_{k,1,T}]
         & \le 1 + \E\left[ \sum_{t=1}^T \1{\mathcal{A}_t^C} \right] +  \beta \cdot \E\left[ \sum_{t=1}^T \1{\mathcal{J}_{k,t} \text{ and } \mathcal{A}_t} \right], \label{eq:bound-M}
        \\
        \E[M_{k,\ell,T}]
         & \le 1 + \E\left[ \sum_{t=1}^T \1{\mathcal{A}_t^C} \right], \quad\forall 1 < \ell \le k, \label{eq:bound-M-ell}
    \end{align}
    where the first constant term comes from the initialization of the algorithm (Algorithm~\ref{alg:initial-phase}), the last terms of~\eqref{eq:bound-N} and~\eqref{eq:bound-M} is due to the imbalanced exploration in Line~\ref{line:defusion-random-decision-choice} of Algorithm~\ref{alg:defusion}, and the missing third term of~\eqref{eq:bound-M-ell} is because when event \(\mathcal{A}_t\) happens, the algorithm does not explore arm pair \((k, \ell)\) for any suboptimal arm \(\ell \neq 1\).

    \begin{align}
         & \quad\, \E\left[ \sum_{t=1}^T \1{\mathcal{J}_{k,t} \text{ and } \mathcal{A}_t} \right]
        \\
         & \le L_k^{\suff} + \sum_{t=1}^T \sum_{n,m > L_k^{\suff}}  \E\left[ \1{\mathcal{J}_{k,t} \text{ and } \mathcal{A}_t \text{ and } N_{k,t} = n, M_{k,1,t} = m }\right]
        \\
         & \le L_k^{\suff} + \sum_{t=1}^T \sum_{n, m > L_k^{\suff}} \E\left[
        \mathbbm{1} \left\{ \max\left\{ I_{k,t}^{\rew}, I_{k,t}^{\duel} - I_{\hat k_t^{\duel},t}^{\duel} \right\} \le \log t + f(K) \text{ and } \mathcal{A}_t \text{ and } N_{k,t} = n, M_{k,1,t} = m \right\}  \right]
        \\\label{eq:defusion-estimate-notation-change}
         & \le L_k^{\suff} + \sum_{n, m > L_k^{\suff}} \E\left[
        \1*{ L_k^{\suff} \max\left\{ \kl(\hat{\mu}_{k}(n), \mu_1 - \delta),  \kl(\hat\nu_{k,1}(m), \frac 1 2)\right\} \le \log T + f(K) }  \right]
        \\\label{eq:defusion-substitute-L-suff}
         & \le L_k^{\suff} + \sum_{n, m > L_k^{\suff}}\E\left[
            \1*{ \max\left\{ \kl(\hat{\mu}_{k,t}(n), \mu_1 - \delta),  \kl(\hat\nu_{k,1,t}(m), \frac 1 2)\right\} \le\frac{\max\left\{\kl(\mu_k, \mu_1),  \kl(\nu_{k,1}, \frac{1}{2})\right\}}{1+\epsilon} }  \right]
        \\
         & \le L_k^{\suff} + \sum_{n > L_k^{\suff}}\E\left[\1*{  \kl(\hat{\mu}_{k,t}(n), \mu_1 - \epsilon) \le\frac{\kl(\mu_k, \mu_1)}{1 + \epsilon} }  \right]
        + \sum_{m > L_k^{\suff}}\E\left[
            \1*{ \kl(\hat\nu_{k,1,t}(m), \frac 1 2) \le\frac{ \kl(\nu_{k,1}, \frac{1}{2})}{1 + \epsilon} }  \right]
        \\
         & \le L_k^{\suff} + O(\epsilon^{-2}), \label{eq:bound-kl-delta-event}
    \end{align}
    where
    inequality~\eqref{eq:defusion-estimate-notation-change} applies the \(\hat \mu_k (n)\) to denote the reward mean estimate with \(n\) samples, and \(\hat \nu_{k,1}(m)\) to denote the dueling probability estimate with \(m\) samples,
    inequality~\eqref{eq:defusion-substitute-L-suff} is by substituting \(L_k^{\suff}\) the definition of \(L_k^{\suff}\),
    % \tdmark~add more detail explanations here.
    inequality~\eqref{eq:bound-kl-delta-event} is for that the last two terms are bounded as \(O(\epsilon^{-2})\) in~\citet[Lemma 15]{honda2010asymptotically} and~\citet[Lemma 6]{komiyama2015regret}, respectively.

    Therefore, the final regret upper bound is given as follows,
    \begin{align}
        \E[\regret]
         & = \E[\regret \vert \mathcal{A}_t] + \E[\regret \vert \mathcal{A}_t^{\text{C}}]
        \\
         & \le \sum_{k\neq 1}  \frac{(\Delta_k^{\rew}/\alpha + \Delta_k^{\duel}/(1-\alpha))((1+\epsilon)\log T + f(K))}{\max\left\{\kl(\mu_k, \mu_1) / \alpha^2, \kl(\nu_{k,1}, \frac{1}{2}) / (1 - \alpha)^2\right\}}
        + O(K^2)  + O(\epsilon^{-2}) + O(e^{cK - f(K)}).
    \end{align}
    where we individual bound the two terms as follows,
    \begin{align}
        \E[\regret \vert \mathcal{A}_t]
         & \le \sum_{k\neq 1} \alpha\Delta_k^{\rew} \E[N_{k,T}]
        + (1-\alpha)\Delta_k^{\duel} \E[M_{k,1, T}]
        \\
         & \le\sum_{k\neq 1}  \frac{(\alpha(1-\beta)\Delta_k^{\rew} + (1-\alpha)\beta\Delta_k^{\duel})((1+\epsilon)\log T + f(K))}{\max\left\{(1-\beta)\kl(\mu_k, \mu_1), \beta \kl(\nu_{k,1}, \frac{1}{2})\right\}}
        + O(\epsilon^{-2})
        \\
         & \le \sum_{k\neq 1}  \frac{(\Delta_k^{\rew}/\alpha + \Delta_k^{\duel}/(1-\alpha))((1+\epsilon)\log T + f(K))}{\max\left\{\kl(\mu_k, \mu_1) / \alpha^2, \kl(\nu_{k,1}, \frac{1}{2}) / (1 - \alpha)^2\right\}}
        + O(\epsilon^{-2}),
    \end{align}
    where the last inequality is by substituting the definition of parameter \(\beta\), and
    \begin{align}
        \E[\regret \vert \mathcal{A}_t^{\text{C}}]
         & \le
        \left( \alpha\Delta_2^{\rew} + (1-\alpha)\max_{k\in\mathcal{K}}\Delta_k^{\duel} \right)  \left( K(K-1) + \E\left[ \sum_{t=1}^T \1{\mathcal{A}_t^C} \right]\right)
        \\
         & \le K^2 + \E\left[ \sum_{t=1}^T \1{\mathcal{A}_t^C} \right]
        \le O(K^2) + O(e^{cK - f(K)}),
    \end{align}
    where the \(K(K-1)\) comes from the initialization of the algorithm (Algorithm~\ref{alg:initial-phase}).
\end{proof}

\section{Experimental Setup and Additional Experiments}\label{app:experiment-setup}

\subsection{Additional Experiments}
\begin{figure}[tb]
    \centering
    \subfigure[Reward regret \(\rregret\)]{
        \includegraphics[width=0.465\columnwidth]{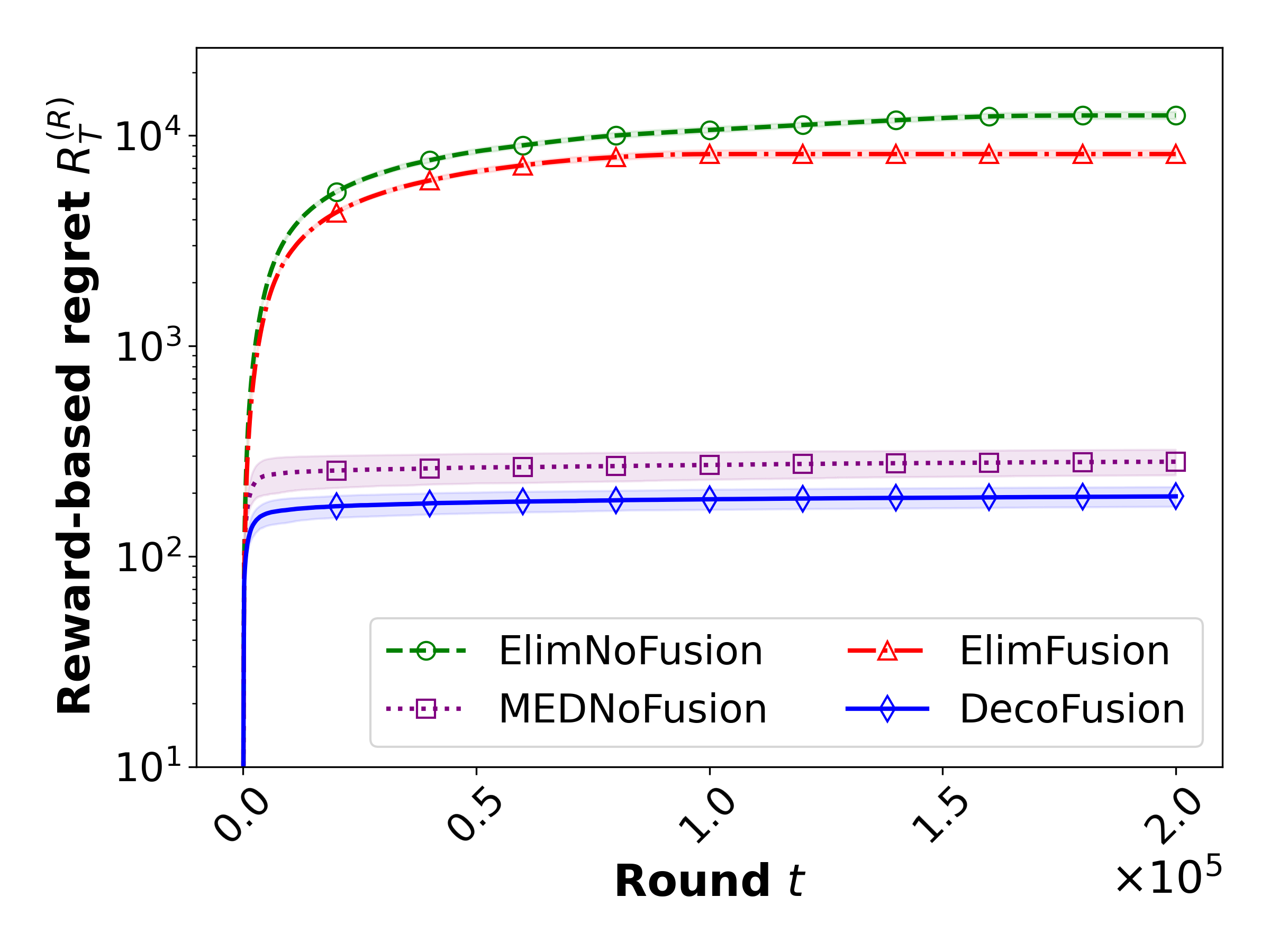}
        \label{fig:reward-regret}
    }
    \hfill
    \subfigure[Dueling regret \(\dregret\)]{
        \includegraphics[width=0.48\columnwidth]{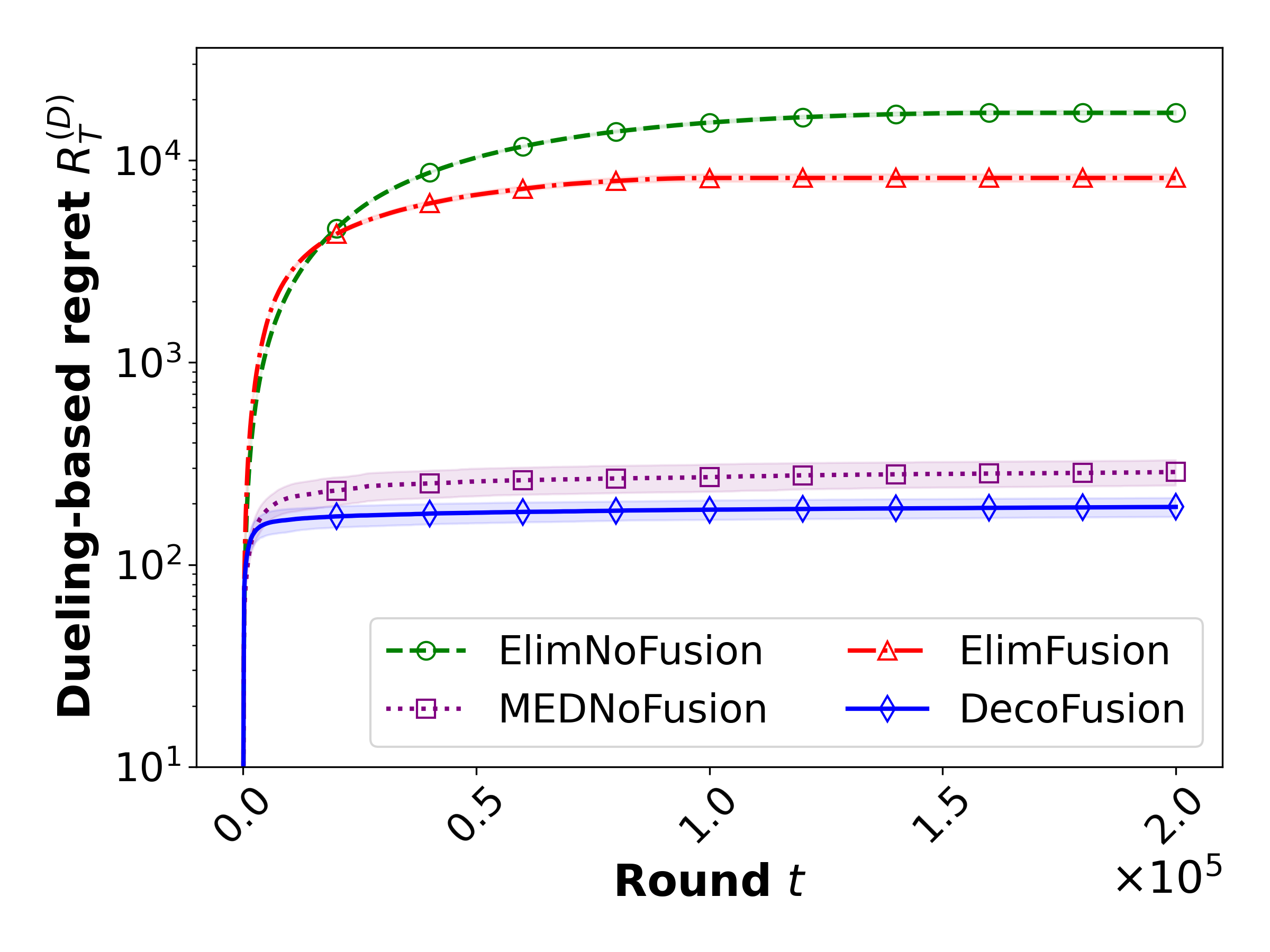}
        \label{fig:dueling-regret}
    }
    % \vfill
    % \subfigure[Compare with Algorithms with Reward Feedback]
    % {
    %     \includegraphics[width=0.47\columnwidth]{figures/new_figures/regret/NO_FUSION-reward.png}
    %     \label{fig:reward-algorithm}
    % }
    % % \tdmark~replace (b) and (c) as a table; (b) \& (c) for alpha = 0.9, 0.1; rounds to round; \(\Delta\) and ; change the ratio of a,b,c
    % \hfill
    % \subfigure[Compare with Algorithms with Dueling Feedback]{\includegraphics[width=0.47\columnwidth]{figures/new_figures/regret/NO_FUSION-dueling.png}
    %     \label{fig:dueling-algorithm}
    % }
    \caption{Regret comparison in different settings (continue)}
    \label{fig:experiments-conti}
\end{figure}

This section further reports other experiments in Figure~\ref{fig:experiments-conti}. Figures~\ref{fig:reward-regret} and~\ref{fig:dueling-regret} plots the decomposed reward- and dueling-based regrets of the aggregated regret in Figure~\ref{fig:final-regret} in the main paper. 
They illustrate that the two fusion algorithms outperform the no-fusion algorithms regarding reward-based and dueling-based regrets, respectively, highlighting the effectiveness of the fusion algorithms in mitigating regret costs to the other feedback.
% Figures~\ref{fig:reward-algorithm} and~\ref{fig:dueling-algorithm} further consider other four baselines: elimination with only reward feedback (Elim (Reward))~\citep{auer2010ucb}, elimination with only dueling feedback (Elim (Dueling))~\citep{saha2021adversarial}, DMED with only reward feedback (DMED (Reward))~\citep{honda2010asymptotically}, and RMED with only dueling feedback (RMED (Dueling))~\citep{komiyama2015regret}.

\subsection{Experiment Setup}

The experiments of Figures~\ref{fig:final-regret},~\ref{fig:alpha-defusion},~\ref{fig:reward-regret},~\ref{fig:dueling-regret}
% ,~\ref{fig:reward-algorithm},~\ref{fig:dueling-algorithm} 
are conducted with $K=16$ arms, where their Bernoulli reward distributions are with means $\bm\mu=\{0.86, 0.80, 0.75, 0.70, 0.65, 0.60, 0.55, 0.50, 0.45, 0.40, 0.35, 0.30, 0.25, 0.20, 0.15, 0.10\}$. % totoal regret and alpha
A dueling probability matrix \(\bm\nu\) determines the dueling feedback as follows,
\[
\left[
\begin{array}{cccccccccccccccc}
0.50& 0.54& 0.57& 0.60& 0.63& 0.65& 0.69& 0.71& 0.73& 0.76& 0.78& 0.82& 0.86& 0.91& 0.95& 0.98\\
0.46& 0.50& 0.54& 0.58& 0.61& 0.64& 0.67& 0.70& 0.74& 0.76& 0.79& 0.81& 0.84& 0.87& 0.89& 0.92\\
0.43& 0.46& 0.50& 0.54& 0.58& 0.60& 0.63& 0.66& 0.69& 0.72& 0.76& 0.79& 0.83& 0.85& 0.88& 0.91\\
0.40& 0.42& 0.46& 0.50& 0.54& 0.58& 0.61& 0.64& 0.66& 0.69& 0.72& 0.76& 0.79& 0.82& 0.85& 0.88\\
0.37& 0.39& 0.42& 0.46& 0.50& 0.54& 0.56& 0.59& 0.63& 0.66& 0.69& 0.72& 0.76& 0.78& 0.82& 0.86\\
0.35& 0.36& 0.40& 0.42& 0.46& 0.50& 0.54& 0.57& 0.59& 0.63& 0.67& 0.70& 0.73& 0.76& 0.79& 0.82\\
0.31& 0.33& 0.37& 0.39& 0.44& 0.46& 0.50& 0.54& 0.58& 0.61& 0.64& 0.68& 0.71& 0.72& 0.75& 0.79\\
0.29& 0.30& 0.34& 0.36& 0.41& 0.43& 0.46& 0.50& 0.54& 0.57& 0.59& 0.62& 0.65& 0.68& 0.72& 0.76\\
0.27& 0.26& 0.31& 0.34& 0.37& 0.41& 0.42& 0.46& 0.50& 0.54& 0.58& 0.61& 0.63& 0.66& 0.69& 0.73\\
0.24& 0.24& 0.28& 0.31& 0.34& 0.37& 0.39& 0.43& 0.46& 0.50& 0.54& 0.56& 0.59& 0.62& 0.66& 0.69\\
0.22& 0.21& 0.24& 0.28& 0.31& 0.33& 0.36& 0.41& 0.42& 0.46& 0.50& 0.54& 0.56& 0.58& 0.61& 0.65\\
0.18& 0.19& 0.21& 0.24& 0.28& 0.30& 0.32& 0.38& 0.39& 0.44& 0.46& 0.50& 0.54& 0.57& 0.58& 0.62\\
0.14& 0.16& 0.17& 0.21& 0.24& 0.27& 0.29& 0.35& 0.37& 0.41& 0.44& 0.46& 0.50& 0.54& 0.56& 0.59\\
0.09& 0.13& 0.15& 0.18& 0.22& 0.24& 0.28& 0.32& 0.34& 0.38& 0.42& 0.43& 0.46& 0.50& 0.54& 0.56\\
0.05& 0.11& 0.12& 0.15& 0.18& 0.21& 0.25& 0.28& 0.31& 0.34& 0.39& 0.42& 0.44& 0.46& 0.50& 0.54\\
0.02& 0.08& 0.09& 0.12& 0.14& 0.18& 0.21& 0.24& 0.27& 0.31& 0.35& 0.38& 0.41& 0.44& 0.46& 0.50
\end{array}
\right]
\]
where the value of $\nu_{i,j}$ between arm pairs \((i,j)\) is in row $i$ column $j$. The algorithms are run for $T=200,000$ rounds with the following parameters for \defusion and \elimfusion: \(\alpha = 0.5\), \(\delta = 1/T\), and \(f(K) = 0.05K^{1.01}\). Each experiment is repeated \(100\) times, and we report the average regret and the standard deviation of all runs.

Then, Figures~\ref{fig:reward-gap} and~\ref{fig:dueling-gap} report the final aggregated regrets under the following two experiments.

\textbf{Fixing $\nu$, varying $\mu$:}
Fixing the dueling probability as in the matrix:
\[
    \bm\nu = 
    \left[
    \begin{array}{ccccc}
        0.50 & 0.53 & 0.56 & 0.59 & 0.62\\
        0.47 & 0.50 & 0.53 & 0.56 & 0.59\\
        0.44 & 0.47 & 0.50 & 0.53 & 0.56\\
        0.41 & 0.44 & 0.47 & 0.50 & 0.53\\
        0.38 & 0.41 & 0.44 & 0.47 & 0.50
    \end{array}
    \right]
\]

Vary $\bm\mu = \{0.9, 0.9 - \Delta, 0.9 - 2\times \Delta, 0.9 - 3 \times \Delta, 0.9 - 4 \times \Delta\}$, where $\Delta \in \{0.06, 0.11, 0.16, 0.21\}$.

\textbf{Fixing $\mu$, varying $\nu$:} Fixing $\bm\mu = \{0.9, 0.84, 0.78, 0.72, 0.66\}$, we consider vary preference matrix in:
\[
\bm \nu = 
\left[
\begin{array}{ccccc}
    0.5 & 0.5 + 1 \times \Delta & 0.5 + 2 \times \Delta & 0.5 + 3 \times \Delta & 0.5 + 4 \times \Delta \\
    0.5 - 1 \times \Delta & 0.5 & 0.5 + 1 \times \Delta & 0.5 + 2 \times \Delta & 0.5 + 3 \times \Delta \\
    0.5 - 2 \times \Delta & 0.5 - 1 \times \Delta & 0.5 & 0.5 + 1 \times \Delta & 0.5 + 2 \times \Delta \\
    0.5 - 3 \times \Delta & 0.5 - 2 \times \Delta & 0.5 - 1 \times \Delta & 0.5 & 0.5 + 1 \times \Delta\\
    0.5 - 4 \times \Delta & 0.5 - 3 \times \Delta & 0.5 - 2 \times \Delta & 0.5 - 1 \times \Delta & 0.5 
\end{array}
\right]
\]
where $\Delta \in \{0.03, 0.05, 0.07, 0.09, 0.11\}$. All other settings of the two experiments are the same as above.